\documentclass[10pt,journal,compsoc]{IEEEtran}
\usepackage{amsmath}
\usepackage{times}
\usepackage{graphicx}
\usepackage{color}
\usepackage{multirow}
\usepackage{subfigure}
\usepackage{stfloats}
\usepackage{graphics}
\usepackage{subfigure}
\usepackage{epstopdf}
\usepackage[figuresright]{rotating}
\usepackage{multicol}
\usepackage{color}
\usepackage{amsmath, amssymb}
\usepackage{amsmath}
\usepackage{amssymb}
\usepackage{amsthm}
\usepackage[ruled, linesnumbered]{algorithm2e}
\usepackage{algorithmic}
\usepackage{setspace}
\usepackage{multirow}
\usepackage{lineno}
\usepackage{array}
\usepackage{rotating}
\usepackage{booktabs}
\usepackage{amsmath}
\newtheorem{theorem}{Theorem}
\newtheorem{lemma}{Lemma}
\newtheorem{definition}{Definition}
\usepackage{algorithm2e,setspace}
\graphicspath{{kbsGraphs/}}
\newtheoremstyle{noparens}%
  {}{}%
  {\itshape}{}%
  {\bfseries}{.}%
  { }%
  {\thmname{#1}\thmnumber{ #2}\mdseries\thmnote{ #3}}
\theoremstyle{noparens}
\newtheorem{theoremNoParens}[theorem]{Theorem}
\ifCLASSOPTIONcompsoc
  \usepackage[nocompress]{cite}
\else
  \usepackage{cite}
\fi
\ifCLASSINFOpdf
\else
\fi

\begin{document}
\begin{sloppypar}
\title{Towards Efficient  Local Causal Structure Learning}
%
%
%
%

\author{Shuai~Yang,~\IEEEmembership{}
        Hao~Wang,~\IEEEmembership{}
        Kui~Yu$^{*}$,~\IEEEmembership{}
        Fuyuan~Cao,~\IEEEmembership{}
        and~Xindong~Wu~\IEEEmembership{}
\IEEEcompsocitemizethanks{\IEEEcompsocthanksitem S. Yang, H. Wang and  K. Yu  are with the Key Laboratory of Knowledge Engineering With Big Data of Ministry of Education, Hefei University of Technology, China, also with the School of Computer Science and Information Engineering, Hefei University of Technology, Hefei 230601, China. E-mail: yangs@mail.hfut.edu.cn; \{jsjxwangh, yukui\}@hfut.edu.cn.
\IEEEcompsocthanksitem F. Cao is with the School of Computer and Information Technology, Shanxi University, Taiyuan 030006, China. E-mail: cfy@sxu.edu.cn.
\IEEEcompsocthanksitem X. Wu is with the Key Laboratory of Knowledge Engineering With Big Data of Ministry of Education, Hefei University of Technology, Hefei 230601,
China, also with Mininglamp Academy of Sciences, Mininglamp Technology, Beijing, 100102, China. E-mail: wuxindong@mininglamp.com.}
\thanks{Manuscript received month day, year; revised month day, year. (*Corresponding author: Kui Yu.)}
}

%
%

\markboth{Journal of \LaTeX\ Class Files,~Vol.~14, No.~8, August~2020}%
{Shell \MakeLowercase{\textit{et al.}}: Bare Demo of IEEEtran.cls for Computer Society Journals}
%



\IEEEtitleabstractindextext{%
\begin{abstract}
Local causal structure learning aims to discover and  distinguish direct causes (parents) and direct effects (children) of a variable of interest from data. While emerging successes have been made, existing methods need to search a large  space to distinguish direct causes from direct effects of a target variable \emph{T}. To tackle this issue, we propose a novel Efficient Local Causal Structure learning algorithm, named ELCS. Specifically, we first propose the concept of N-structures, then design an efficient Markov Blanket (MB) discovery subroutine to integrate MB learning with N-structures to learn the MB of \emph{T} and simultaneously distinguish  direct causes from direct effects of \emph{T}. With the proposed MB subroutine, ELCS starts from the target variable, sequentially finds MBs of variables connected to the target variable and simultaneously constructs local causal structures over MBs until the direct causes and direct effects of the target variable have been distinguished. Using eight Bayesian networks the extensive experiments have validated that ELCS achieves better accuracy and efficiency than the state-of-the-art algorithms.
\end{abstract}

}

\maketitle

\IEEEdisplaynontitleabstractindextext

%
\IEEEpeerreviewmaketitle

\IEEEraisesectionheading{\section{Introduction}\label{sec:introduction}}
\IEEEPARstart{C}AUSAL discovery has always been an important goal  in many scientific fields, such as medicine, computer science and bioinformatics \cite{YuTKDD20,CaiQZZH19,Huang0GG19,MarxV19a}. There has been a great deal of recent interest in discovering causal relationships between variables, since it is not only helpful  to reveal the underlying data generating mechanism, but also to improve classification  and prediction  performance in both static and non-static environments  \cite{yu2019multi}. However, in many real-world scenarios, it is difficult to  discover causal relationships between variables since true causality can only be identified using  controlled experimentation \cite{GourevitchBF06}.

\par Statistical approaches are useful in generating testable causal hypotheses which can accelerate the causal discovery process \cite{ChoiCN20}.  Learning a Bayesian network (BN) from observational data is the popular method for causal structure learning and causal inference. The structure of a BN takes the form of a directed acyclic graph (DAG) in which nodes of the DAG represent the variables and edges represent dependence between variables. A DAG implies causal concepts, since they code potential causal relationships between variables: the existence of a directed edge \emph{X}$\rightarrow$\emph{Y} means that \emph{X} is a direct cause of \emph{Y}, and the absence of a directed edge \emph{X}$\rightarrow$\emph{Y} means that \emph{X} cannot be a direct cause of \emph{Y} \cite{maathuis2009}. When a directed edge \emph{X}$\rightarrow$\emph{Y} in a BN indicates that \emph{X} is a direct cause of \emph{Y}, in this case, the BN is known as a causal Bayesian network. Given a set of conditional dependencies from observational data and a corresponding DAG model, we can infer a causal Bayesian network using intervention calculus \cite{pearl2009causality}. Then learning BN structures (i.e. DAGs) from observational data is the most important step for causal structure learning.

\par In recent years, many  causal structure learning  (i.e. DAG learning) methods have been designed \cite{MarxV19b}, which can be roughly divided into  global causal  structure learning and local causal structure learning. The first type of methods aims to learn the casual structure of all variables, such as MMHC \cite{TsamardinosBA06}, NOTEARS \cite{ZhengARX18} and DAG-GNN \cite{YuCGY19}. However, in many practical scenarios, it is not necessary to waste time  to learn a global structure when we are only interested in the causal relationships around a given  variable.  To tackle this issue, the second type of methods is proposed, with the aim to discover and distinguish the direct causes (parents) and direct effects (children) of a target variable, such as  PCD-by-PCD \cite{YinZWHZG08} and CMB \cite{GaoJ15}.

PCD-by-PCD (PCD means Parents, Children and some Descendants) \cite{YinZWHZG08} and  CMB (Causal Markov Blanket) \cite{GaoJ15} first learn the PCD  or MB (Markov Blanket) of a target variable and construct a local structure among the target variable and the variables in the PCD or MB, then sequentially learn PCDs or MBs of the variables connected to the target variable and simultaneously construct local structures among variables in  PCDs or MBs until the parents  and children  of the target variable have been distinguished.

\par While emerging successes have been made, existing local causal structure learning methods suffer from the following limitations. They need to search a large space to distinguish parents from children of a target variable. That is to say, existing local causal structure learning methods not only need to learn the PCD or MB of the target variable, but also may need to learn  PCDs or MBs of the variables connected to the target variable. In the worst case (e.g. the target variable has all single ancestors) all existing methods may be required to learn PCDs or MBs of all variables in a dataset. This leads to that existing local causal structure learning methods are often computationally expensive or even infeasible especially with a large-sized BN. For instance, as shown in Fig. \ref{Example1}, there is an N-structure (see Definition \ref{def6}) formed by four variables \emph{T}, \emph{A}, \emph{B}, and \emph{C}. Given the target variable \emph{T}, in order to determine the causal relationship between \emph{T} and \emph{B}: PCD-by-PCD is required to learn the PCD of \emph{T} and PCDs of \emph{A} and \emph{B}. For CMB, if only the MB of \emph{T} is learnt, the edge direction between \emph{T} and \emph{B} cannot be determined since there is no V-structures around B.  CMB needs to further learn the MBs of \emph{B} and \emph{A} to orient the edge between \emph{T} and \emph{B}.  In a word, both PCD-by-PCD and  CMB need to search a large space to determine the edge direction between \emph{T} and \emph{B}. A larger search space will result in performing more conditionally independence (CI) tests for discovering the causal relationships around a given variable. More CI tests not only increase computational time, but also lead to more unreliable tests. It will be beneficial to  local causal structure learning if  the  edge direction between \emph{B} and \emph{T} can be determined in learning the PCD or MB of the target variable \emph{T} without learning PCDs or MBs of the other variables.

\begin{figure}
  \centering
  \includegraphics[width=4.5cm]{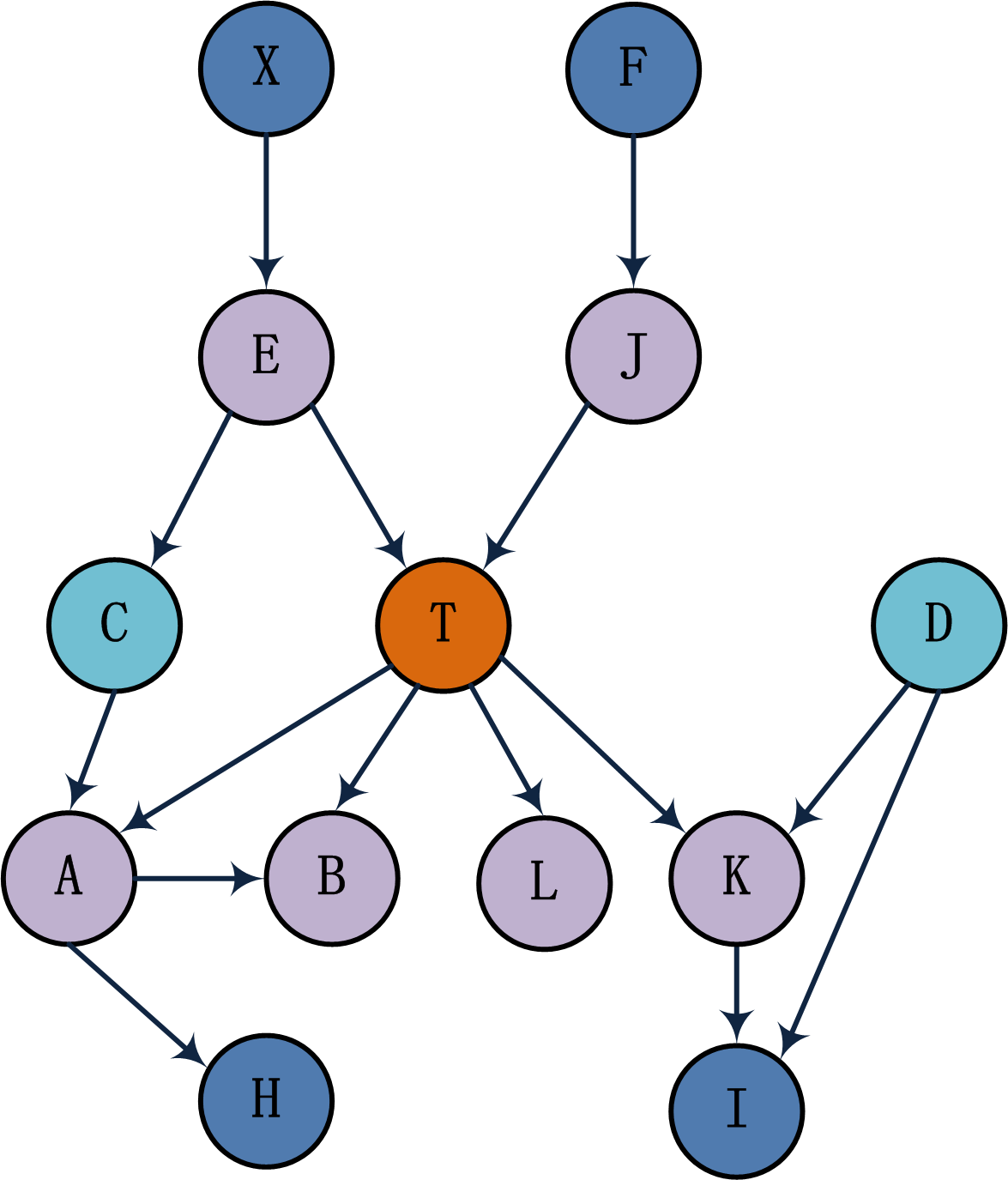}
  \caption{A sample Bayesian network. The MB  of \emph{T} includes \emph{E} and \emph{J} (parents), \emph{A}, \emph{B}, \emph{L} and \emph{K} (children), \emph{C} and \emph{D} (spouses).}
  \label{Example1}
\end{figure}

\par Then a question naturally arises: can we reduce the search space in determining the edge directions between a given variable and its children to speed up the local causal structure learning? To address this problem, our main contributions of the paper are summarized as follows.
\begin{itemize}
\item We propose the concept of N-structures, a special local structure for edge directions in local causal structure learning. Then we propose a new local causal structure learning, called ELCS. Through leveraging the N-structures, ELCS learns the MBs of the variables as few as possible to distinguish parents from children of a given variable as many as possible, which improves the efficiency of local causal structure learning and simultaneously reduces the impact of unreliable CI tests.
\item  To integrate MB learning with N-structures to infer edge directions as many as possible during the MB learning procedure, we design an efficient MB discovery subroutine (EMB) and its efficient version EMB-II. EMB not only is able to learn the MB of a variable, but also has an ability to distinguish parents from children of the variable.
\item  We have conducted extensive experiments on eight benchmark BNs, and have compared ELCS with five existing causal structure learning algorithms, including three state-of-the-art global structure learning and two local structure learning algorithms, to demonstrate the effectiveness and efficiency of the ELCS algorithm.
\end{itemize}
\par The remainder of this paper is organized as follows. Section 2 reviews the related work, and  Section 3 gives the notations and definitions. Section 4  describes the proposed ELCS algorithm in detail. Section 5 reports and discusses the experimental results. Section 6 summarizes the paper.

\section{Related Work}
\label{relatedwork}
Our work focuses on local causal structure learning and is also related to MB learning and global causal structure learning. So this section briefly introduces the related work in the three areas.

\par \textbf{MB learning}. Learning Markov Blanket (MB) plays an essential part  in the skeleton learning during BN structure learning. Existing MB learning methods can be categorized into two types: constraint-based methods and score-based methods. The former employs  independence tests  to find the MB of  a given variable \cite{CCMB,BAMB}, whereas the latter learns the MB using score-based BN structure learning  algorithms \cite{NiinimakiP12,S2TMB}.

\par Constraint-based methods can be roughly grouped into simultaneous MB learning and divide-and-conquer MB learning. Given the target variable \emph{T}, the simultaneous MB learning algorithm aims to learn  parents, children, and spouses of \emph{T} simultaneously, and does not distinguish  spouses of \emph{T} from its PC, such as  GSMB \cite{GSMB}, IAMB \cite{IAMB}, Inter-IAMB \cite{tsamardinos2003algorithms} and Fast-IAMB \cite{FastIAMB}.  To reduce the  sample requirement of the simultaneous MB learning algorithm, the divide-and-conquer MB learning algorithm is proposed, with the aim to find PC and spouses of the target variable separately. The representative divide-and-conquer MB learning algorithms include CCMB \cite{CCMB}, BAMB \cite{BAMB}, MMMB \cite{TsamardinosAS03}, HITON-MB \cite{AliferisTS03}, PCMB \cite{PCMB} and STMB \cite{GaoJ17}. Recently, a comprehensive review of the state-of-the-art MB learning algorithms are discussed in \cite{yu2020causality}.

\par However, existing MB learning methods only learn a local skeleton around a target variable and do not distinguish parents from children in the learnt MB of a target variable.
\par \textbf{Global causal structure learning}. A large amount of methods have been designed for global causal structure learning. Recent methods can be roughly categorized into two types:  local-to-global structure learning methods and continuous optimization based learning methods. The local-to-global structure learning approach, such as MMHC \cite{TsamardinosBA06}, SSL$\pm$C/G \cite{NiinimakiP12} and GSBN \cite{MargaritisT99}, first learns the MB or PC  of each variable, then constructs a skeleton of a DAG using  the  learnt MBs or PCs, and finally  orients edges of  the learnt skeleton using score-based or constraint-based causal learning algorithms.  Instead of learning the MB of each variable first, GGSL \cite{GaoFC17} starts with a randomly selected variable, and then  uses a score-based MB learning algorithm to gradually expand the learnt structure through a series of local structure learning steps. Based on GGSL, a parallel BN structure learning algorithm (PSL) is designed to improve the efficiency \cite{GaoW18}.

\par Recently,  several continuous optimization based learning approaches  have been proposed for global causal structure learning \cite{YuCGY19,ZhengARX18,Zhu1906,DVAE}. Zheng et al. consider a BN structure learning  problem as a purely continuous optimization problem and propose the NOTEARS algorithm \cite{ZhengARX18}. DAG-GNN uses a  graph neural network based  deep generative model to capture the complex data distribution to learn BN structures \cite{YuCGY19}. RL-BIC uses reinforcement learning to search for a directed acyclic graph (DAG) with the best score \cite{Zhu1906}.  Zhang et al. propose a DAG variational autoencoder (D-VAE) for BN structure learning \cite{DVAE}.

\par However, global causal structure learning methods are time consuming or even infeasible when the number of variables of a BN is large. In fact, in many practical settings, we are only interested in distinguishing parents from children of a variable of interest. In this case, it is unnecessary and wasteful to find an entire BN structure.

\par \textbf{Local causal structure learning}. Local causal structure learning aims to learn and distinguish the parents and children of a target variable. Although many algorithms have been designed for learning a whole structure, only several algorithms have been proposed for local causal structure learning. PCD-by-PCD first discovers the PCD of a target variable, then sequentially discovers  PCDs of the variables connected  to the target variable and  simultaneously finds V-structures and orients the edges  connected to the target variable until  all the  parents and children of the target variable are identified \cite{YinZWHZG08}.  CMB first learns the MB of a target variable using HITON-MB and orients edges  by tracking the conditional independence changes in the MB of the target variable, then  sequentially learns MBs of the variables connected  to the target variable and  simultaneously  construct local structures along the paths starting from the target variable until  the parents and children of the target variable have been identified or  they cannot be identified further by continuing the process \cite{GaoJ15}.

\par As we discussed in Section 1, both PCD-by-PCD and CMB encounter the time inefficient problem since they need to learn a large number of PCDs or MBs of the variables for distinguishing parents from children of a target variable. To  tackle this issue, in this paper, we aim to develop a new method through learning the MBs of variables as few as possible while orienting edges as many as possible.

\section{Notations and Definitions}
In this section, we will briefly introduce some basic definitions and notations frequently used in this paper (see Table \ref{notation} for a summary of the notations).  Let $\emph{\textbf{U}}$ denote  a set of random variables. $\mathbb{P}$   represents a joint probability distribution over $\emph{\textbf{U}}$, and $\mathbb{G}$   is a  DAG over $\emph{\textbf{U}}$. In a  DAG, $\emph{X}$ is a parent of  $\emph{Y}$ and $\emph{Y}$ is a child of $\emph{X}$ if there exists a directed edge from  $\emph{X}$ to  $\emph{Y}$.  $\emph{X}$ is an ancestor of $\emph{Y}$ (i.e., non-descendant of $\emph{Y}$) and $\emph{Y}$ is a descendant of $\emph{X}$ if there exists a directed path from $\emph{X}$ to $\emph{Y}$.

\tabcolsep 0.05in
\begin{table}
\scriptsize
\caption{Summary of  Notations}\label{notation}
\centering
\begin{tabular}{|c|c|}
\hline
Notation &Meanings              \\
\hline
\emph{\textbf{U}}    &  a set of random variables\\\hline
\emph{\textbf{W}}    &  a subset of \emph{\textbf{U}}\\\hline
$\mathbb{P}$           &  a joint probability distribution over \emph{\textbf{U}}\\\hline
$\mathbb{G}$             &  a direct acyclic graph over \emph{\textbf{U}} \\\hline
DAG          &  direct acyclic graph \\\hline
\emph{X}, \emph{Y}, \emph{Z}, \emph{T }            &  a single variable in  \emph{\textbf{U}}\\\hline
\emph{\textbf{Z}},\emph{\textbf{S} }          &  a conditioning set within  \emph{\textbf{U}} \\\hline
\emph{X}$\!\perp\!\!\!\perp$\emph{Y}      &  \emph{X} and \emph{Y}  are  independent given \emph{\textbf{Z}} \\\hline
\emph{X}$\not\!\perp\!\!\!\perp$\emph{Y}      &  \emph{X} and \emph{Y}  are  dependent given \emph{\textbf{Z}} \\\hline
\emph{\textbf{MB}}$_{\emph{T}}$     &  Markov Blanket of \emph{T} \\\hline
\textbf{\emph{PC}}$_{\emph{T}}$     &  a set of parents and children of \emph{T} \\\hline
\textbf{\emph{P}}$_{\emph{T}}$      &  a set of parents  of \emph{T} \\\hline
\textbf{\emph{C}}$_{\emph{T}}$      &  a set of children of \emph{T}   \\\hline
\textbf{\emph{UN}}$_{\emph{T}}$      &  undistinguished variables in \textbf{\emph{PC}}$_{\emph{T}}$\\\hline
\textbf{\emph{SP}}$_{\emph{T}}$     &   a set of spouses of \emph{T} \\\hline
\textbf{\emph{SP}}$_{\emph{T}}$\{{\emph{X}}\}    &  a  spouses of \emph{T} with regard to \emph{T}'s child \emph{X} \\\hline
\textbf{\emph{Sep}}$_{\emph{T}}$\{{\emph{X}}\}    &  a set that \emph{d}-separates \emph{X} from \emph{T}\\\hline
\textbf{\emph{Sep}}$_{\emph{T}}$    &  a set that contains the  sets \textbf{\emph{Sep}}$_{\emph{T}}$\{{$\cdot$}\}  of all variables\\\hline
\textbf{\emph{CSP}}$_{\emph{T}}$    &  a set that contains the candidate spouse sets  of  all \emph{PC}$_{\emph{T}}$ variables\\\hline
\emph{Que}          &  a circular queue(first in fist out)\\\hline
$|\cdot|$         &  the size of a set\\\hline
\end{tabular}
\end{table}
\begin{definition}[Conditional Independence \cite{neufeld1993pearl}]\label{def1}
Given a conditioning set \textbf{Z},   X is conditionally independent of  Y  if and only if $P(X|Y,\textbf{Z}) = P(X|\textbf{Z})$.
\end{definition}

\begin{definition}[Bayesian Network \cite{neufeld1993pearl}]\label{BN}
 The triplet $<$\textbf{U},$\mathbb{G}$,$\mathbb{P}$$>$  is called a Bayesian network (BN)  if $<$\textbf{U},$\mathbb{G}$,$\mathbb{P}$$>$ satisfies the Markov condition: each variable is conditionally independent of  variables in its non-descendant given its parents in $\mathbb{G}$.
\end{definition}

\begin{definition}[Casual Bayesian Network \cite{pearl2009causality}]\label{BN}
A BN is called a causal Bayesian network (CBN) if a directed edge in $\mathbb{G}$ has causal interpretation, that is, X$\rightarrow$Y indicates that X is a direct cause of Y.
\end{definition}

\begin{definition}[Causal Structure Learning]\label{csl}
Global causal structure learning   aims to learn a DAG  over \textbf{U} from observational data, where  edges represent potential causal relationships between variables, that is, X is a direct cause of Y if there exists a directed edge from X to Y \emph{\cite{pearl2009causality}}. Local causal structure learning  aims to  discover and distinguish direct causes and direct effects  of a  variable of interest \emph{\cite{GaoJ15}}.
\end{definition}

\begin{definition}[V-structure \cite{neufeld1993pearl}]\label{defVS}
If there is no an edge between X and Y, and Z has two incoming edges from X and Y, respectively, then X, Z and Y form a V-structure ($X\rightarrow Z \leftarrow Y$).
\end{definition}


In a BN, \emph{Z} is  a collider if there are two directed edges from  \emph{X} to \emph{Z} and from  \emph{Y} to \emph{Z}, respectively. V-structures play an important role in determining  the edge directions between variables. For example, if there is a  V-structure ($X\rightarrow Z \leftarrow Y$) formed by \emph{X}, \emph{Y} and \emph{Z}, we can identify  \emph{X} and \emph{Y} as parents of \emph{Z} using  conditional independence (CI) tests.

\begin{definition}[D-separation \cite{neufeld1993pearl}]\label{def3}
Given a set  \textbf{S} $\subseteq $  \textbf{U}$\setminus$\{X,Y\}, a path $\pi$ between  X and Y is blocked, if one of the following conditions  is satisfied: 1) there is a non-collider variable within \textbf{S} on $\pi$, or 2) there is a collier variable Z  on $\pi$, while Z and any  its descendants are not in \textbf{S}. Otherwise, $\pi$ between  X and Y is  unblocked. X and Y are d-separation given \textbf{S}  if and only if  each path between X and Y is blocked by \textbf{S}.
\end{definition}

In a DAG, given a conditioning set, we can determine whether two variables are conditionally independent using  Definition \ref{def3}.

\begin{definition}[Faithfulness \cite{CPS}]\label{def4}
Given a BN $<$\textbf{U},$\mathbb{G}$,$\mathbb{P}$$>$,  $\mathbb{G}$ is faithful to $\mathbb{P}$ if and only if all the conditional independencies appear in $\mathbb{P}$ are entailed by $\mathbb{G}$.  $\mathbb{P}$ is faithful if and only if there is a \emph{DAG} $\mathbb{G}$ such that $\mathbb{G}$ is faithful to $\mathbb{P}$.
\end{definition}

Definition \ref{def4} indicates that in a faithful BN, if \emph{X} and \emph{Y} are d-separated given the conditioning set \emph{\textbf{S}} in $\mathbb{G}$, then they will be conditionally independent  given \emph{\textbf{S}} in $\mathbb{P}$.

\begin{definition}[Markov Blanket \cite{neufeld1993pearl}]\label{def5}
In a faithful BN, the MB of a target variable T is denoted as \textbf{MB}$_{T}$, which is uniqueness and consists of parents, children and spouses (other parents of the target's children) of T. All other variables  in \textbf{U} $\setminus$\textbf{MB}$_{T}$$\setminus$\{T\} are conditionally independent of T given \textbf{MB}$_{T}$, $\forall$ X $\subseteq$  \textbf{U} $\setminus$\textbf{MB}$_{T}$$\setminus$\{T\}, X $\!\perp\!\!\!\perp$ T $|$ \textbf{MB}$_{T}$, where X $\!\perp\!\!\!\perp$ T $|$ \textbf{MB}$_{T}$ denotes X and T are conditionally independent  conditioning on \textbf{MB}$_{T}$.
\end{definition}

\begin{definition}[N-structure]\label{def6}
In a faithful BN, if there exists four variables T, A, B and C, and T is a parent of A and B, C is a parent of A, there is no an edge between C and T, A is an ancestor of B, the other parents of B are in  PC set of T. Then, the local structure formed by the four variables  is called an N-structure.
\end{definition}

\begin{figure}
  \centering
  \includegraphics[width=6cm]{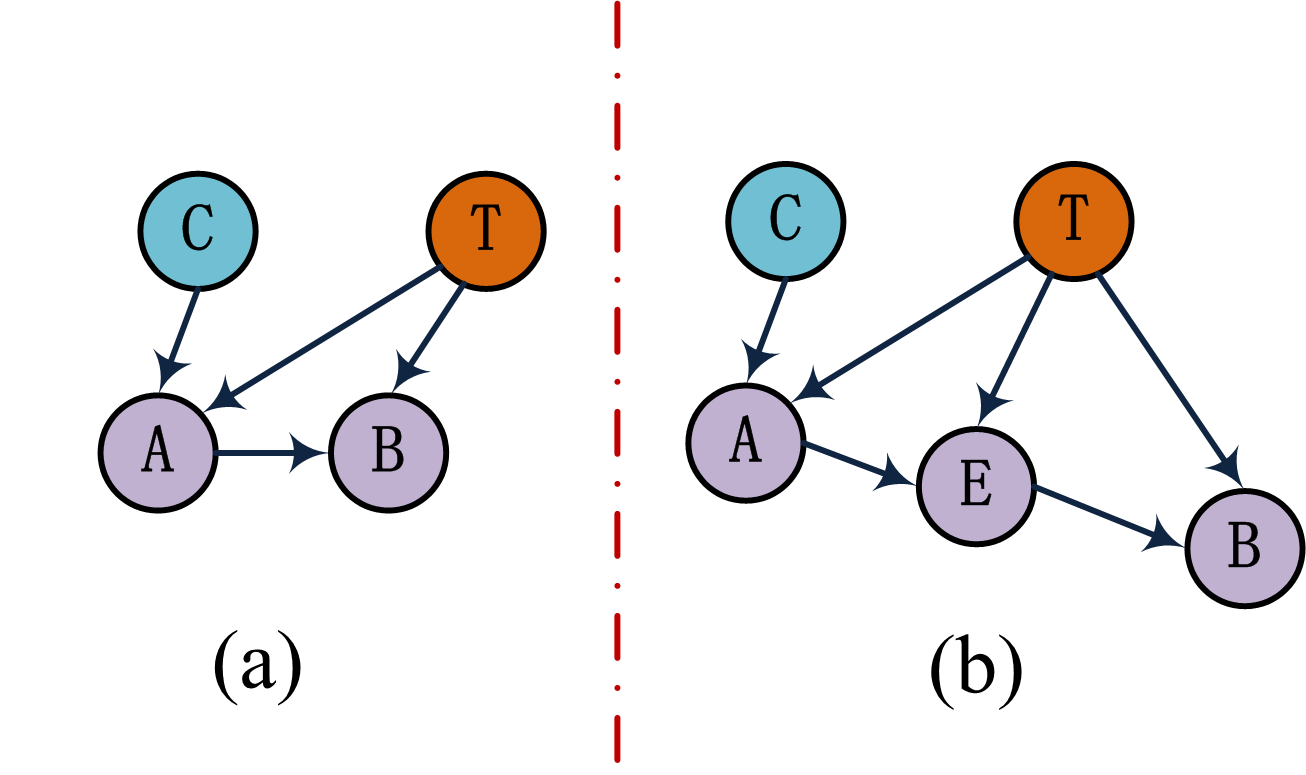}
  \caption{Examples of N-structures.}
  \label{Examplesub1}
\end{figure}

Fig. \ref{Examplesub1} gives examples of the N-structures. In Fig. \ref{Examplesub1} (a),  there is an N-structure formed by \emph{T}, \emph{A}, \emph{B} and \emph{C}. In Fig. \ref{Examplesub1} (b), variables \emph{T}, \emph{A}, \emph{B} and \emph{C} construct an N-structure, and  variables \emph{T}, \emph{A}, \emph{E} and \emph{C} construct an N-structure. Given the target variable \emph{T}, we can leverage the N-structures to determine edge directions between \emph{T} and its children (i.e. \emph{B} and \emph{E}) during learning the MB of \emph{T} without learning MBs of the other variables.

\begin{theoremNoParens}[\cite{CPS}]\label{thm1}
In a faithful BN, for any two variables X $\in$ \textbf{U} and Y $\in$ \textbf{U}, if there exists an edge between X and Y, then $\forall$ \textbf{S} $\subseteq$  \textbf{U} $\setminus$ \{X,Y\}, X$\not\!\perp\!\!\!\perp$Y $|$ \textbf{S} holds.
\end{theoremNoParens}

Theorem \ref{thm1} demonstrates that if \emph{X} is a parent (or a child) of \emph{Y} if \emph{X} and \emph{Y} are not conditionally independent conditioning any subsets excluding \emph{X} and \emph{Y}.

\section{The  Proposed  Method}

\subsection{The ELCS Algorithm}
We propose the Efficient Local Causal Structure learning  algorithm (ELCS) to distinguish  parents  from children  of a target variable, as shown in Algorithm \ref{algorithmLCDMB}. ELCS starts from the target variable, sequentially finds MBs of variables connected to the target variable and simultaneously constructs local causal structures over MBs until all the parents and children of the target variable have been distinguished or it is clear that they cannot be further distinguished  by continuing the process. In the following, we first summarize the main idea  of ELCS, then give the details of ELCS.

\par To improve the efficiency of local causal structure learning, in ELCS, we propose the following two acceleration strategies. First, ELCS finds the N-structures, and then leverages those found N-structures to infer  edge directions between the target variable \emph{T} and its children during learning the MB of \emph{T}. Second, two rules in Lemma 1 are used to further infer edge directions between \emph{T} and its PC during learning the MB of \emph{T}.

\par As described in Algorithm \ref{algorithmLCDMB}, given the target variable \emph{T}, ELCS first initializes the variable set \emph{\textbf{W}} and the queue \emph{Que} to  empty sets (line 1 in Algorithm \ref{algorithmLCDMB}), where \emph{\textbf{W}}  is used to  store variables that their MBs have been learnt, and  \emph{Que}  is utilized to  store variables that their MBs need to be learnt in  next phase. Then, \emph{T} enters  \emph{Que} (line 2 in Algorithm \ref{algorithmLCDMB}). Next, lines 4-13 in Algorithm \ref{algorithmLCDMB} will be executed. At line 4 in Algorithm  \ref{algorithmLCDMB}, the header element in \emph{Que}  is out of queue, that is, \emph{X} = \emph{T}. Since the MB of \emph{T} has not been learnt, \emph{T} is added to the set \emph{\textbf{W}} (lines 5-6 in Algorithm \ref{algorithmLCDMB}). The EMB (Efficient Markov Blanket discovery) subroutine is executed to learn the MB of \emph{T} (line 7 in Algorithm \ref{algorithmLCDMB}). Given a  variable \emph{X}, the EMB subroutine not only is able to find the PC (\textbf{\emph{PC}}$_{\emph{X}}$) and MB (\textbf{\emph{SP}}$_{\emph{X}}$ $\cup$ \textbf{\emph{PC}}$_{\emph{X}}$), but also has an ability to distinguish parents from children of \emph{X}. The details of EMB  are described in Section 4.2. Let \textbf{\emph{P}}$_{\emph{X}}$ represent a set of the identified parents of \emph{X}. \textbf{\emph{C}}$_{\emph{X}}$ denotes a set of the identified children of \emph{X}.  The set  containing undistinguished PC variables of \emph{X}  is denoted as \textbf{\emph{UN}}$_{\emph{X}}$.  After executing the line 7 in Algorithm \ref{algorithmLCDMB}, the sets of \textbf{\emph{P}}$_{\emph{T}}$, \textbf{\emph{C}}$_{\emph{T}}$, \textbf{\emph{UN}}$_{\emph{T}}$, \textbf{\emph{SP}}$_{\emph{T}}$, \textbf{\emph{PC}}$_{\emph{T}}$ are obtained.  If \textbf{\emph{UN}}$_{\emph{T}}$ is  empty, that is, parents and children of \emph{T} are all distinguished, the learning process will be terminated. Otherwise, the undistinguished variables within \textbf{\emph{UN}}$_{\emph{T}}$ will be put in  \emph{Que} (lines 9-11 in Algorithm  \ref{algorithmLCDMB}). Then, Meek rules \cite{Meek1995Causal} are used to orient other edge directions between variables in \emph{\textbf{W}} (line 13 in Algorithm  \ref{algorithmLCDMB}).  Lines 4-13 in Algorithm \ref{algorithmLCDMB} will be repeated until all the parents  and children  of \emph{T} have been distinguished or \emph{Que} is empty or the size of \emph{\textbf{W}} equals to that of the entire variable set.
\begin{algorithm}[t]
\footnotesize
\caption{ELCS}
\label{algorithmLCDMB}
\begin{algorithmic}[1]
\REQUIRE ~~
Data \emph{D}, target variable \emph{T }, random variables set \textbf{\emph{U}}\\
\ENSURE ~~
\textbf{\emph{P}}$_{\emph{T}}$, \textbf{\emph{C}}$_{\emph{T}}$, \textbf{\emph{UN}}$_{\emph{T}}$\\
\STATE \emph{\textbf{W}} $\leftarrow$ $\emptyset$, \emph{Que} = $\emptyset$\\
\STATE \emph{Que}$.push$(\emph{T})\\
\STATE \textbf{Repeat}\\
\STATE \quad \emph{X} = \emph{Que}.$pop()$
\STATE \quad \textbf{if} \emph{X} $\not\in$ \emph{\textbf{W}} \textbf{then}
\STATE \quad \quad \emph{\textbf{W}} $\leftarrow$  \emph{\textbf{W}} $\cup$ \{\emph{X}\}
\STATE \quad \quad [\textbf{\emph{P}}$_{\emph{X}}$, \textbf{\emph{C}}$_{\emph{X}}$, \textbf{\emph{UN}}$_{\emph{X}}$, \textbf{\emph{SP}}$_{\emph{X}}$, \textbf{\emph{PC}}$_{\emph{X}}$] = EMB(\emph{D}, \emph{X})
\STATE \quad \quad Orienting edge directions between \emph{X} and its PC nodes \\according to  \textbf{\emph{P}}$_{\emph{X}}$ and \textbf{\emph{C}}$_{\emph{X}}$
\STATE \quad \quad \textbf{for }each \emph{Y} $\in$ \textbf{\emph{UN}}$_{\emph{X}}$  \textbf{do}
\STATE \quad \quad \quad  \emph{Que}$.push$(\emph{Y})
\STATE \quad \quad \textbf{end for}
\STATE \quad \textbf{end if}
\STATE \quad Using Meek rules to orient other edge directions between\\ variables in \emph{\textbf{W}}
\STATE \textbf{Until} (1) all parents and children of \emph{T} can be determined, or \\(2) \emph{Que} = $\emptyset$, or (3)\emph{\textbf{W}} = \emph{\textbf{U}}\\
\end{algorithmic}
\end{algorithm}

\begin{algorithm}[t]
\footnotesize
\caption{EMB}
\label{algorithmEMB}
\begin{algorithmic}[1]
\REQUIRE ~~
Data \emph{D}, target variable \emph{T }\\
\ENSURE ~~
\textbf{\emph{P}}$_{\emph{T}}$, \textbf{\emph{C}}$_{\emph{T}}$, \textbf{\emph{UN}}$_{\emph{T}}$, \textbf{\emph{SP}}$_{\emph{T}}$, \textbf{\emph{PC}}$_{\emph{T}}$\\
/*\emph{Step 1: find the PC set of T}*/\\
\STATE [\textbf{\emph{PC}}$_{\emph{T}}$, \textbf{\emph{Sep}}$_{\emph{T}}$]$\leftarrow$ RecogPC(\emph{T}, \emph{D})\\
/*\emph{Step 2: find spouses of T}*/\\
\STATE [\textbf{\emph{SP}}$_{\emph{T}}$, \textbf{\emph{CSP}}$_{\emph{T}}$]$\leftarrow$ RecogSpouses(\emph{D}, \emph{T}, \textbf{\emph{PC}}$_{\emph{T}}$, \textbf{\emph{Sep}}$_{\emph{T}}$)\\
/*\emph{Step 3: remove  false positive PC from  \textbf{PC}$_{T}$}*/\\
\STATE  \textbf{for }each \emph{Y} $\in$ \textbf{\emph{PC}}$_{\emph{T}}$ \textbf{do}
\STATE \quad  \textbf{if} \emph{T} $\!\perp\!\!\!\perp$ \emph{Y} $|$ \emph{\textbf{Z}},  \emph{\textbf{Z}} $\subseteq$ \textbf{\emph{SP}}$_{\emph{T}}$\{{\emph{Y}}\} $\cup$   \textbf{\emph{PC}}$_{\emph{T}}$ $\setminus$ \{\emph{Y}\}    \textbf{then}
\STATE \quad \quad  \textbf{\emph{PC}}$_{\emph{T}}$ $\leftarrow$ \textbf{\emph{PC}}$_{\emph{T}}$ $\setminus$ \{\emph{Y}\}
\STATE \quad \quad  \textbf{\emph{SP}}$_{\emph{T}}$\{{\emph{Y}}\} $\leftarrow$ $\emptyset$
\STATE \quad   \textbf{end if}
\STATE  \textbf{end for}\\
/*\emph{Step 4: find \textbf{P}$_{T}$, \textbf{C}$_{T}$}*/\\
\STATE  [\textbf{\emph{P}}$_{\emph{T}}$, \textbf{\emph{C}}$_{\emph{T}}$, \textbf{\emph{UN}}$_{\emph{T}}$]$\leftarrow$ DistinguishPC(\emph{D}, \emph{T}, \textbf{\emph{PC}}$_{\emph{T}}$, \textbf{\emph{SP}}$_{\emph{T}}$, \textbf{\emph{CSP}}$_{\emph{T}}$)
\end{algorithmic}
\end{algorithm}

\begin{algorithm}[t]
\footnotesize
\caption{RecogSpouses}
\label{FindSpouses}
\begin{algorithmic}[1]
\REQUIRE ~~
Data \emph{D}, target variable \emph{T}, \textbf{\emph{PC}}$_{\emph{T}}$, \textbf{\emph{Sep}}$_{\emph{T}}$ \\
\ENSURE ~~
\textbf{\emph{SP}}$_{\emph{T}}$, \textbf{\emph{CSP}}$_{\emph{T}}$\\
\STATE  \textbf{\emph{SP}}$_{\emph{T}}$ $\leftarrow$ $\emptyset$
\STATE \textbf{for} each \emph{X} $\in$ \emph{\textbf{U}}$\setminus$\{\emph{T}\}$\setminus$\textbf{\emph{PC}}$_{\emph{T}}$ \textbf{do}
\STATE \quad \textbf{\emph{Temp}} $\leftarrow$ $\emptyset$
\STATE \quad \textbf{for} each \emph{Y} $\in$ \textbf{\emph{PC}}$_{\emph{T}}$ \textbf{do}
\STATE \quad \quad \textbf{if} \emph{X} $\not\!\perp\!\!\!\perp$ \emph{Y} $|$ $\emptyset$ \textbf{then}
\STATE \quad\quad \quad \textbf{\emph{Temp}} $\leftarrow$ \textbf{\emph{Temp}} $\cup$  \{\emph{Y}\}
\STATE \quad\quad \textbf{end if}
\STATE \quad \textbf{end for}
\STATE \quad \textbf{if} \emph{X} $\not\!\perp\!\!\!\perp$ \emph{T} $|$ \textbf{\emph{Temp}} \textbf{then}
\STATE \quad\quad \textbf{for} each \emph{Y} $\in$ \textbf{\emph{Temp}} \textbf{do}
\STATE \quad\quad\quad  \textbf{if} \emph{X} $\not\!\perp\!\!\!\perp$ \emph{T} $|$ \{\emph{Y}\} $\cup$  \textbf{\emph{Sep}}$_{\emph{T}}$\{\emph{X}\} \textbf{then}
\STATE \quad\quad\quad\quad \textbf{\emph{CSP}}$_{\emph{T}}$\{{\emph{Y}}\} $\leftarrow$ \textbf{\emph{CSP}}$_{\emph{T}}$\{{\emph{Y}}\} $\cup$ \{\emph{X}\}
\STATE \quad\quad\quad \textbf{end if}
\STATE \quad\quad \textbf{end for}
\STATE \quad \textbf{end if}
\STATE \textbf{end for}
\STATE \textbf{\emph{SP}}$_{\emph{T}}$ $\leftarrow$ \textbf{\emph{CSP}}$_{\emph{T}}$
\STATE  \textbf{for} each \emph{Y} $\in$ \textbf{\emph{PC}}$_{\emph{T}}$ \textbf{do}
\STATE  \quad \textbf{for} each \emph{X} $\in$ \textbf{\emph{SP}}$_{\emph{T}}$\{{\emph{Y}}\} \textbf{do}
\STATE \quad \quad  \textbf{if} \emph{X} $\!\perp\!\!\!\perp$ \emph{Y} $|$ \emph{\textbf{Z}}, \emph{\textbf{Z}} $\subseteq$ \textbf{\emph{SP}}$_{\emph{T}}$\{{\emph{Y}}\} $\cup$ \{\emph{T}\} $\cup$  \textbf{\emph{PC}}$_{\emph{T}}$ $\setminus$ \{\emph{X},\emph{Y}\}    then
\STATE \quad \quad \quad  \textbf{\emph{SP}}$_{\emph{T}}$\{{\emph{Y}}\} $\leftarrow$ \textbf{\emph{SP}}$_{\emph{T}}$\{{\emph{Y}}\} $\setminus$ \{\emph{X}\}
\STATE \quad \quad  \textbf{end if}
\STATE \quad \textbf{end for}
\STATE \textbf{end for}\\
\end{algorithmic}
\end{algorithm}

\begin{algorithm}[t]
\footnotesize
\caption{DistinguishPC}
\label{algorithmIdentifyPC}
\begin{algorithmic}[1]
\REQUIRE ~~
Data \emph{D}, target variable \emph{T }, \textbf{\emph{PC}}$_{\emph{T}}$, \textbf{\emph{SP}}$_{\emph{T}}$, \textbf{\emph{CSP}}$_{\emph{T}}$ \\
\ENSURE ~~
\textbf{\emph{P}}$_{\emph{T}}$, \textbf{\emph{C}}$_{\emph{T}}$, \textbf{\emph{UN}}$_{\emph{T}}$\\
\STATE  \textbf{\emph{C}}$_{\emph{T}}$  $\leftarrow$ $\emptyset$, \textbf{\emph{P}}$_{\emph{T}}$  $\leftarrow$ $\emptyset$
\STATE  \textbf{for} each \emph{Y} $\in$ \textbf{\emph{PC}}$_{\emph{T}}$ \textbf{do}
\STATE \quad  \textbf{if} \textbf{\emph{SP}}$_{\emph{T}}$\{{\emph{Y}}\} is nonempty \textbf{then}
\STATE \quad \quad  \textbf{\emph{C}}$_{\emph{T}}$  $\leftarrow$ \textbf{\emph{C}}$_{\emph{T}}$ $\cup$ \{\emph{Y}\}
\STATE \quad   \textbf{end if}
\STATE  \textbf{end for}
\STATE  \textbf{\emph{UN}}$_{\emph{T}}$ $\leftarrow$ \textbf{\emph{PC}}$_{\emph{T}}$  $\setminus$ \textbf{\emph{C}}$_{\emph{T}}$
\STATE  \textbf{for} each \emph{X} $\in$ \textbf{\emph{UN}}$_{\emph{T}}$ \textbf{do}
\STATE \quad  \textbf{if} \textbf{\emph{CSP}}$_{\emph{T}}$\{{\emph{X}}\} $\cap$ \textbf{\emph{C}}$_{\emph{T}}$ is nonempty \textbf{then}
\STATE \quad \quad  \textbf{\emph{C}}$_{\emph{T}}$  $\leftarrow$ \textbf{\emph{C}}$_{\emph{T}}$ $\cup$ \{\emph{X}\}
\STATE \quad   \textbf{end if}
\STATE  \textbf{end for}
\STATE  \textbf{for} each \emph{X} $\in$ \textbf{\emph{PC}}$_{\emph{T}}$  $\setminus$ \textbf{\emph{C}}$_{\emph{T}}$ \textbf{do}
\STATE  \quad \textbf{for} each \emph{Y} $\in$ \textbf{\emph{PC}}$_{\emph{T}}$  $\setminus$ \textbf{\emph{C}}$_{\emph{T}}$ \textbf{do}
\STATE \quad  \quad \textbf{if} \emph{X}  $\!\perp\!\!\!\perp$ \emph{Y} $|$ $\emptyset$ and \emph{X}  $\not\!\perp\!\!\!\perp$ \emph{Y} $|$ \emph{T} \textbf{then}
\STATE \quad  \quad \quad \quad \textbf{\emph{P}}$_{\emph{T}}$ $\leftarrow$ \textbf{\emph{P}}$_{\emph{T}}$ $\cup$ \{\emph{X}\} $\cup$ \{\emph{Y}\}
\STATE \quad \quad  \textbf{end if}
\STATE \quad   \textbf{end for}
\STATE  \textbf{end for}
\STATE  \textbf{\emph{UN}}$_{\emph{T}}$ $\leftarrow$ \textbf{\emph{PC}}$_{\emph{T}}$ $\setminus$ \textbf{\emph{P}}$_{\emph{T}}$ $\setminus$ \textbf{\emph{C}}$_{\emph{T}}$
\STATE  \textbf{for} each \emph{X} $\in$ \textbf{\emph{UN}}$_{\emph{T}}$ \textbf{do}
\STATE  \quad \textbf{for} each \emph{Y} $\in$ \textbf{\emph{P}}$_{\emph{T}}$ \textbf{do}
\STATE \quad \quad  \textbf{if} \emph{X}  $\not\!\perp\!\!\!\perp$ \emph{Y} $|$ $\emptyset$ and \emph{X}  $\!\perp\!\!\!\perp$ \emph{Y} $|$ \emph{T} \textbf{then}
\STATE \quad \quad \quad \textbf{\emph{C}}$_{\emph{T}}$ $\leftarrow$ \textbf{\emph{C}}$_{\emph{T}}$ $\cup$ \{\emph{X}\}
\STATE \quad \quad \quad \textbf{break}
\STATE \quad  \quad \textbf{end if}
\STATE \quad \textbf{end for}
\STATE \textbf{end for}
\STATE  \textbf{\emph{UN}}$_{\emph{T}}$ $\leftarrow$ \textbf{\emph{PC}}$_{\emph{T}}$ $\setminus$ \textbf{\emph{P}}$_{\emph{T}}$ $\setminus$ \textbf{\emph{C}}$_{\emph{T}}$
\end{algorithmic}
\end{algorithm}

\subsection{EMB  Subroutine}
\par  ELCS depends on the MB learning methods for local causal structure learning, but existing MB learning algorithms have the following shortcomings. First, existing MB learning methods cannot be directly combined with the N-structures to infer edge directions between the target variable and its children. Second, existing MB learning methods only focus on learning the MB of the target variable and are not able to distinguish parents from children. Third, existing MB learning methods may be computationally expensive. In order to help ELCS to leverage N-structures and the rules in Lemma \ref{lemma1} for efficiently learning local causal structures, we design an Effective MB discovery subroutine (EMB) to learn the MB of a target variable and distinguish parents from children of the target variable simultaneously.

 \par As shown in Algorithm \ref{algorithmEMB}, EMB consists of four steps as follows. Given a target variable \emph{T}, EMB first learns the PC  of \emph{T} (\emph{\textbf{PC}}$_{\emph{T}}$) using an existing PC learning algorithm. Second, EMB obtains  spouses of \emph{T} using a RecogSpouses subroutine. Then, EMB removes false PC  from \emph{\textbf{PC}}$_{\emph{T}}$. Finally, EMB orients edges between \emph{T} and its PC  as many as possible using a DistinguishPC subroutine. Specifically, to find the N-structures, EMB first determines which variable within \emph{\textbf{U}}$\setminus$\{\emph{T}\}$\setminus$\textbf{\emph{PC}}$_{\emph{T}}$ is a candidate spouse of \emph{T}, and obtains the candidate spouse set \textbf{\emph{CSP}}$_{\emph{T}}$\{{\emph{Y}}\} of each variable \emph{Y} within \textbf{\emph{PC}}$_{\emph{T}}$, and then obtains the spouses of \emph{T}. Based on the learnt \textbf{\emph{CSP}}$_{\emph{T}}$\{{\emph{Y}}\} and  spouses, we can find the N-structures. Through leveraging the found N-structures,  EMB can distinguish some children of \emph{T} with regard to the found N-structures.  In addition, two rules in Lemma \ref{lemma1} are used in the DistinguishPC subroutine to further distinguish  parents from children of \emph{T}.  In the following, we will give the details of these four steps.

 \begin{lemma}\label{lemma1}
The PC (parents and children) set of a given variable T (T $\in$  \textbf{U}) is denoted as \textbf{PC}$_{T}$. Let X $\in$ \textbf{PC}$_{T}$, Y $\in$ \textbf{PC}$_{T}$.  We can get the following two dependence relationships between  X and  Y.
\par\emph{(a)} X $\!\perp\!\!\!\perp$ Y $|$ $\emptyset$ and X  $\not\!\perp\!\!\!\perp$ Y $|$ T $\Rightarrow$  X and  Y are both parents of  T. This shows that there is a V-structure ($X\rightarrow T \leftarrow Y$) formed by variables X, Y and T, and T is a collider.
\par\emph{(b)} X is a direct cause of T, X  $\not\!\perp\!\!\!\perp$ Y $|$ $\emptyset$ and X  $\!\perp\!\!\!\perp$ Y $|$ T $\Rightarrow$  Y is a direct effects of  T. This shows that there is only one path ($X \rightarrow T \rightarrow Y$) from X to Y, and the path is blocked by T.
\end{lemma}

\par Step 1 (line 1 in Algorithm \ref{algorithmEMB}): EMB obtains \textbf{\emph{PC}}$_{\emph{T}}$ and \textbf{\emph{Sep}}$_{\emph{T}}$ of a target variable \emph{T} by utilizing an existing PC learning algorithm, where \textbf{\emph{Sep}}$_{\emph{T}}$ is a set that contains the  sets \textbf{\emph{Sep}}$_{\emph{T}}$\{{$\cdot$}\}  of all variables. In this paper, we use HITON-PC \cite{AliferisTS03} to find  the PC  of  \emph{T} (any other state-of-the-art PC learning algorithms can be used here to instantiate the RecogPC() function at Step 1 in Algorithm   \ref{algorithmEMB}).

\par Step 2 (line 2 in Algorithm \ref{algorithmEMB}): At this step, EMB learns spouses of \emph{T}. We design a RecogSpouses subroutine for learning spouses. The details of RecogSpouses are described in Algorithm \ref{FindSpouses}. RecogSpouses first finds candidate spouses from all variables within \emph{\textbf{U}}$\setminus$\{\emph{T}\}$\setminus$\textbf{\emph{PC}}$_{\emph{T}}$ that are conditionally independent of \emph{T}. If  \emph{X} and \emph{T} are conditionally independence, then  we  construct a set \textbf{\emph{Temp}} that consists of variables which belong to \emph{\textbf{PC}}$_{\emph{T}}$ and  are  dependent of \emph{X} given an empty set (lines 3-8 in Algorithm   \ref{FindSpouses}). If \emph{X}  and \emph{T} are conditionally independent conditioning on \textbf{\emph{Temp}}, then \emph{X} cannot be a spouse of \emph{T}. Otherwise, \emph{X} is regarded as a candidate spouse and lines 9-15 in Algorithm \ref{FindSpouses} will be executed. If \emph{X} and \emph{T} are dependent conditioning on \textbf{\emph{Sep}}$_{\emph{T}}$\{\emph{X}\} $\cup$ \emph{Y} (\emph{Y} $\in$ \textbf{\emph{Temp}}), then \emph{X} will be added to \textbf{\emph{CSP}}$_{\emph{T}}$\{\emph{Y}\} (lines 9-15 in Algorithm \ref{FindSpouses}). Since some non-parent variables of  \emph{Y}  will be added to \textbf{\emph{CSP}}$_{\emph{T}}$\{\emph{Y}\}, non-parent variables of each \emph{Y}  $\in$ \textbf{\emph{PC}}$_{\emph{T}}$ will be removed  from \textbf{\emph{CSP}}$_{\emph{T}}$\{\emph{Y}\} and  the spouse set \textbf{\emph{SP}}$_{\emph{T}}$\{\emph{Y}\} will be obtained after executing lines 17-24 in Algorithm \ref{FindSpouses}.

\par Step 3 (lines 3-8 in Algorithm \ref{algorithmEMB}): At this step, EMB removes false positives from the candidate set of PC of \emph{T}. For each  variable \emph{Y} within \textbf{\emph{PC}}$_{\emph{T}}$, if there exists a subset \textbf{\emph{Z}} of the union \textbf{\emph{SP}}$_{\emph{T}}$\{\emph{Y}\} $\cup$ \textbf{\emph{PC}}$_{\emph{T}}$ such that \emph{Y} and  \emph{T} are conditionally independent conditioning on \textbf{\emph{Z}}, \emph{Y} will be removed from \textbf{\emph{PC}}$_{\emph{T}}$, and \textbf{\emph{SP}}$_{\emph{T}}$\{\emph{Y}\} will be set to an empty set.

\par Step 4 (line 9 in Algorithm \ref{algorithmEMB}): At this step,  EMB distinguishes  parents from children of \emph{T} as many as possible. We propose a DistinguishPC subroutine to accomplish this goal. DistinguishPC first identifies some children of \emph{T} with the help of spouses of  \emph{T}. Second, DistinguishPC uses  the found N-structures to infer  edge directions between \emph{T} and its children. Finally, DistinguishPC distinguishes  parents from children of \emph{T} using Lemma 1.

\par The details of DistinguishPC  are described in Algorithm \ref{algorithmIdentifyPC}. First, DistinguishPC uses the learnt spouses to  identify some children of \emph{T} (lines 2-6 in Algorithm \ref{algorithmIdentifyPC}). For example, in Fig. \ref{Example1}, \emph{C} and \emph{D} are spouses of \emph{T}, \textbf{\emph{SP}}$_{\emph{T}}$\{{\emph{A}}\} = \{\emph{C}\}, \textbf{\emph{SP}}$_{\emph{T}}$\{{\emph{K}}\} = \{\emph{D}\}, and DistinguishPC identifies \emph{A} and \emph{K} as children of \emph{T}.  There exists an N-structure which is formed by \emph{C}, \emph{A}, \emph{B} and \emph{T} in Fig. \ref{Example1}, DistinguishPC  can determine the edge direction between \emph{T} and \emph{B} with the help of \textbf{\emph{SP}}$_{\emph{T}}$ and \textbf{\emph{CSP}}$_{\emph{T}}$. At Step 2, \emph{C} will be added to \textbf{\emph{CSP}}$_{\emph{T}}$\{{\emph{A}}\} and \textbf{\emph{CSP}}$_{\emph{T}}$\{{\emph{B}}\}, and \emph{C} will be removed from \textbf{\emph{SP}}$_{\emph{T}}$\{{\emph{B}}\} because \emph{C} is not a parent of \emph{B} (lines 17-24 in Algorithm \ref{FindSpouses}). Since \emph{C} is a spouse of \emph{T} and \emph{C} is within the set \textbf{\emph{CSP}}$_{\emph{T}}$\{{\emph{B}}\}, DistinguishPC identifies \emph{B} as a child of \emph{T} (lines 8-12 in Algorithm \ref{algorithmIdentifyPC}). Theorem \ref{thm2} gives the theoretical analysis. In addition, in order to orient more edges between \emph{T} and its PC, two rules in Lemma 1 are used.  If \emph{X} $\!\perp\!\!\!\perp$ \emph{Y} $|$ $\emptyset$ and \emph{X}  $\not\!\perp\!\!\!\perp$ \emph{Y} $|$ \emph{T}, we can conclude that  \emph{X} and  \emph{Y} are both parents of \emph{T}. Therefore, DistinguishPC identifies  parents of \emph{T} as many as possible using Lemma \ref{lemma1} (a) (lines 13-19 in Algorithm \ref{algorithmIdentifyPC}).  If \emph{X} is a parent of \emph{T}, \emph{X}  $\not\!\perp\!\!\!\perp$ Y $|$ $\emptyset$ and \emph{X}  $\!\perp\!\!\!\perp$ \emph{Y} $|$ \emph{T}, then  \emph{Y} is a child of  \emph{T}.  DistinguishPC uses the identified parents of \emph{T} to  determine  edge directions between \emph{T} and its children using Lemma \ref{lemma1} (b) (lines 21-28 in Algorithm \ref{algorithmIdentifyPC}).

\par We also propose a variant of EMB to further improve the efficiency of MB learning, which is referred to as EMB-II. Compared with EMB, in learning MB, instead of directly executing line 20 in Algorithm \ref{FindSpouses}, EMB-II  first ranks the variables within \textbf{\emph{SP}}$_{\emph{T}}$\{{\emph{Y}}\} in descending order according to the dependency with variable \emph{Y}, then executes line 20 in Algorithm \ref{FindSpouses}.
 \par We also propose a variant of ECLS to further improve the efficiency of local causal structure learning, which is referred to as ELCS-II. Compared with ECLS, ELCS-II uses EMB-II for MB learning instead of using EMB.

 \par  In the following, we will give the details of Theorem 2 and its proof.

\begin{theorem}\label{thm2}
In a faithful BN, given an N-structure consisting of  four variables T, A, B, C. T is a parent of A and B, C is a parent of A, there is no an edge between C and T, A is an ancestor of B, and the parents of B are in PC set of T, then EMB identifies B as a child of T during learning the MB of  T.
\end{theorem}

\begin{proof}
\par Under the faithfulness assumption, the PC of a target variable  \emph{T} only contains parents and children of \emph{T}. Since  \emph{C} is a parent of  \emph{A}, there is no an edge between \emph{C} and \emph{T}, then EMB identifies \emph{C} as a spouse of \emph{T} since there is a V-structure ($\emph{C}\rightarrow \emph{A} \leftarrow \emph{T}$) around  \emph{A}.  At step 2 of EMB,  \emph{C} will be added to \textbf{\emph{CSP}}$_{\emph{T}}$\{{\emph{B}}\} (lines 2-16 in Algorithm \ref{FindSpouses}) since \emph{C} and \emph{T} are conditionally dependent given the conditioning set \{\emph{B}\} $\cup$  \textbf{\emph{Sep}}$_{\emph{T}}$\{\emph{C}\}. If \emph{B} is a parent of \emph{T}, then  \emph{C} and \emph{T} are conditionally independent given the conditioning set \{\emph{B}\} $\cup$  \textbf{\emph{Sep}}$_{\emph{T}}$\{\emph{C}\}, since all paths from  \emph{T} to \emph{C} are blocked by conditioning set \{\emph{B}\} $\cup$  \textbf{\emph{Sep}}$_{\emph{T}}$\{\emph{C}\}. Therefore, \emph{B} is a child of \emph{T}.
\end{proof}

\subsection{Tracing}
We first trace the execution of EMB using the example in Fig. \ref{executexample}, then trace the execution of ELCS using the example in Fig. \ref{executexample2}.

\subsubsection{Tracing EMB}
\label{TraEMB}

\begin{figure*}
  \centering
  \includegraphics[width=18cm]{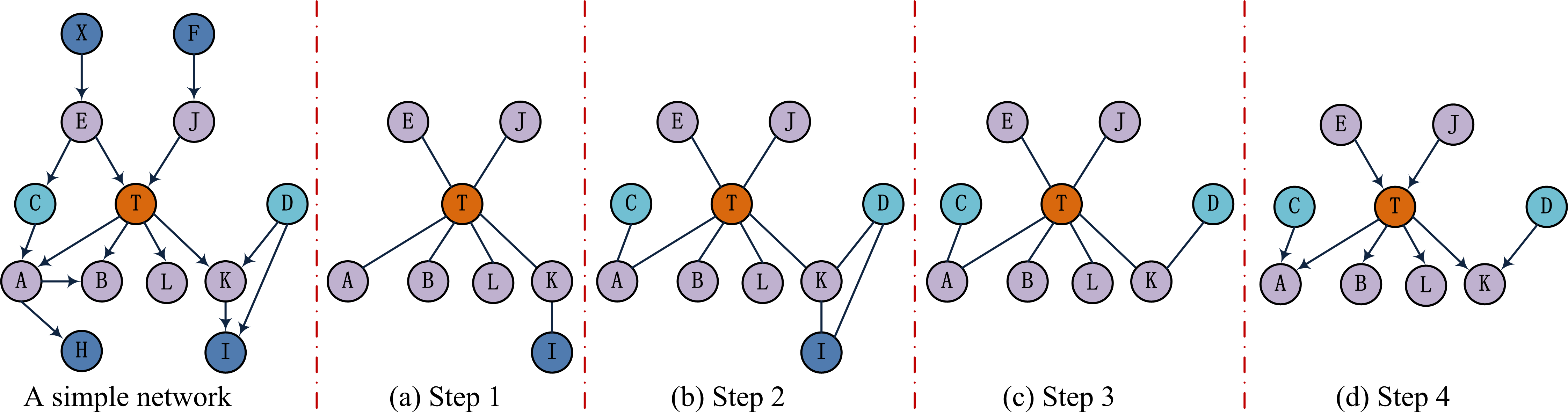}
  \caption{An example of the execution process of EMB.}
  \label{executexample}
\end{figure*}

\par We utilize the example  in Fig. \ref{executexample} to trace the execution of EMB. Suppose that we have a dataset for variable set \textbf{\emph{U} } = \{\emph{A}, \emph{B}, \emph{C}, \emph{D}, \emph{E}, \emph{F}, \emph{X}, \emph{H}, \emph{I}, \emph{J}, \emph{K}, \emph{L}, \emph{T}\}. The independence relationships between variables can be represented by the Bayesian Network structure in Fig. \ref{executexample}. In the following, we regard \emph{T} as the target variable, and give the execution process of EMB.
\par (1) step 1: referring to the simple network, i.e., the left network in Fig. \ref{executexample}. First,  HITON-PC is used to find the PC of \emph{T}. According to Theorem 1, \emph{A}, \emph{B}, \emph{L}, \emph{K}, \emph{E} and \emph{J} will be added to \textbf{\emph{PC}}$_{\emph{T}}$. Note that \emph{D} is conditionally independent of  \emph{T} given an empty set, hence  \emph{D} will not be in any of the conditioning sets for higher order conditional independence tests. As a result, \emph{I} will be added to \textbf{\emph{PC}}$_{\emph{T}}$ since \emph{T} and \emph{I} are conditionally dependent given conditioning set \{\emph{K}\}. Then, as shown in Fig. \ref{executexample} (a), \textbf{\emph{PC}}$_{\emph{T}}$ = \{\emph{A}, \emph{B}, \emph{L}, \emph{K}, \emph{E}, \emph{J}, \emph{I}\}.
\par (2) step 2: as shown in Fig.\ref{executexample} (b), EMB discovers the spouses of \emph{T}. \emph{X} and each variable within \textbf{\emph{Temp}} = \{\emph{A}, \emph{B}, \emph{L}, \emph{K}, \emph{E}, \emph{I}\} are conditionally dependent given an empty set, while \emph{X} is conditionally independent of \emph{T} given the conditioning set \textbf{\emph{Temp}}, so that \emph{X} cannot be a candidate spouse of \emph{T} since each path from \emph{X} to \emph{T} is blocked by \textbf{\emph{Temp}}. Similarly, both \emph{F} and \emph{H} are not candidate spouses.  \emph{C} and each variable within \textbf{\emph{Temp}} = \{\emph{A}, \emph{B}, \emph{L}, \emph{K}, \emph{E}, \emph{I}\} are conditionally dependent given an empty set, and \emph{C} is dependent of \emph{T} given  \textbf{\emph{Temp}}, and we need to conduct further tests. Owning to \emph{C}  $\!\perp\!\!\!\perp$ \emph{T} $|$ \emph{E}, \emph{C}  $\not\!\perp\!\!\!\perp$ \emph{T} $|$ \{\emph{E}, \emph{A}\}, \emph{C}  $\not\!\perp\!\!\!\perp$ \emph{T} $|$ \{\emph{E}, \emph{B}\}, \emph{C}  $\!\perp\!\!\!\perp$ \emph{T} $|$ \{\emph{E}, \emph{L}\}, \emph{C}  $\!\perp\!\!\!\perp$ \emph{T} $|$ \{\emph{E}, \emph{K}\} and \emph{C}  $\!\perp\!\!\!\perp$ \emph{T} $|$ \{\emph{E}, \emph{I}\}, hence  \emph{C} is added to  \textbf{\emph{CSP}}$_{\emph{T}}$\{{\emph{A}}\} and \textbf{\emph{CSP}}$_{\emph{T}}$\{{\emph{B}}\}, \textbf{\emph{CSP}}$_{\emph{T}}$\{{\emph{A}}\} = \{\emph{C}\}, \textbf{\emph{CSP}}$_{\emph{T}}$\{{\emph{B}}\} = \{\emph{C}\}.  Similarly,  \emph{D} is added to \textbf{\emph{CSP}}$_{\emph{T}}$\{{\emph{K}}\} and  \textbf{\emph{CSP}}$_{\emph{T}}$\{{\emph{I}}\}, \textbf{\emph{CSP}}$_{\emph{T}}$\{{\emph{K}}\} = \{\emph{D}\}, \textbf{\emph{CSP}}$_{\emph{T}}$\{{\emph{I}}\} = \{\emph{D}\}. In the following, \emph{C} will be removed from \textbf{\emph{SP}}$_{\emph{T}}$\{{\emph{B}}\} since \emph{C}  $\!\perp\!\!\!\perp$ \emph{B} $|$ \{\emph{A}, \emph{T}\} (lines 17-24 in Algorithm \ref{FindSpouses}).  After this step, \textbf{\emph{SP}}$_{\emph{T}}$\{{\emph{A}}\} = \{\emph{C}\}, \textbf{\emph{SP}}$_{\emph{T}}$\{{\emph{K}}\} = \{\emph{D}\}, \textbf{\emph{SP}}$_{\emph{T}}$\{{\emph{I}}\} = \{\emph{D}\}.
\par (3) step 3: as shown in Fig. \ref{executexample} (c), after checking at line 4 in Algorithm \ref{algorithmEMB}, \emph{I} will be removed from \textbf{\emph{PC}}$_{\emph{T}}$ since \emph{I}  $\!\perp\!\!\!\perp$ \emph{T} $|$ \{\emph{K}, \emph{D}\}. After this step, \textbf{\emph{PC}}$_{\emph{T}}$ = \{\emph{A}, \emph{B}, \emph{L}, \emph{K}, \emph{E}, \emph{J}\}, \textbf{\emph{SP}}$_{\emph{T}}$ = \{\emph{C}, \emph{D}\}, and \emph{\textbf{MB}}$_{\emph{T}}$ =\{\emph{A}, \emph{B}, \emph{L}, \emph{K}, \emph{E}, \emph{J}, \emph{C}, \emph{D}\}.
\par (4) step 4: as shown in Fig. \ref{executexample} (d), EMB orients the edge directions between  \emph{T} and its PC  as many as possible.  Since \emph{C} is a spouse of \emph{T}, and \emph{C} has been added to  \textbf{\emph{SP}}$_{\emph{T}}$\{{\emph{B}}\} at step 2, based on Theorem \ref{thm2}, \emph{B} is a child of \emph{T}. In addition, according to Lemma 1, \emph{E} and \emph{J} are parents of \emph{T} since \emph{E}  $\!\perp\!\!\!\perp$ \emph{J} $|$ $\emptyset$ and \emph{E}  $\not\!\perp\!\!\!\perp$ \emph{J} $|$ \emph{T}. \emph{L} is a child of \emph{T} since \emph{E} is a parent of \emph{T}, \emph{E}  $\not\!\perp\!\!\!\perp$ \emph{L} $|$ $\emptyset$ and \emph{E}  $\!\perp\!\!\!\perp$ \emph{L} $|$ \emph{T}.

\subsubsection{Tracing ELCS}
\label{TraELCS}
We use the example in Fig. \ref{executexample2} to trace the execution of ELCS.
\begin{figure*}
  \centering
  \includegraphics[width=12cm]{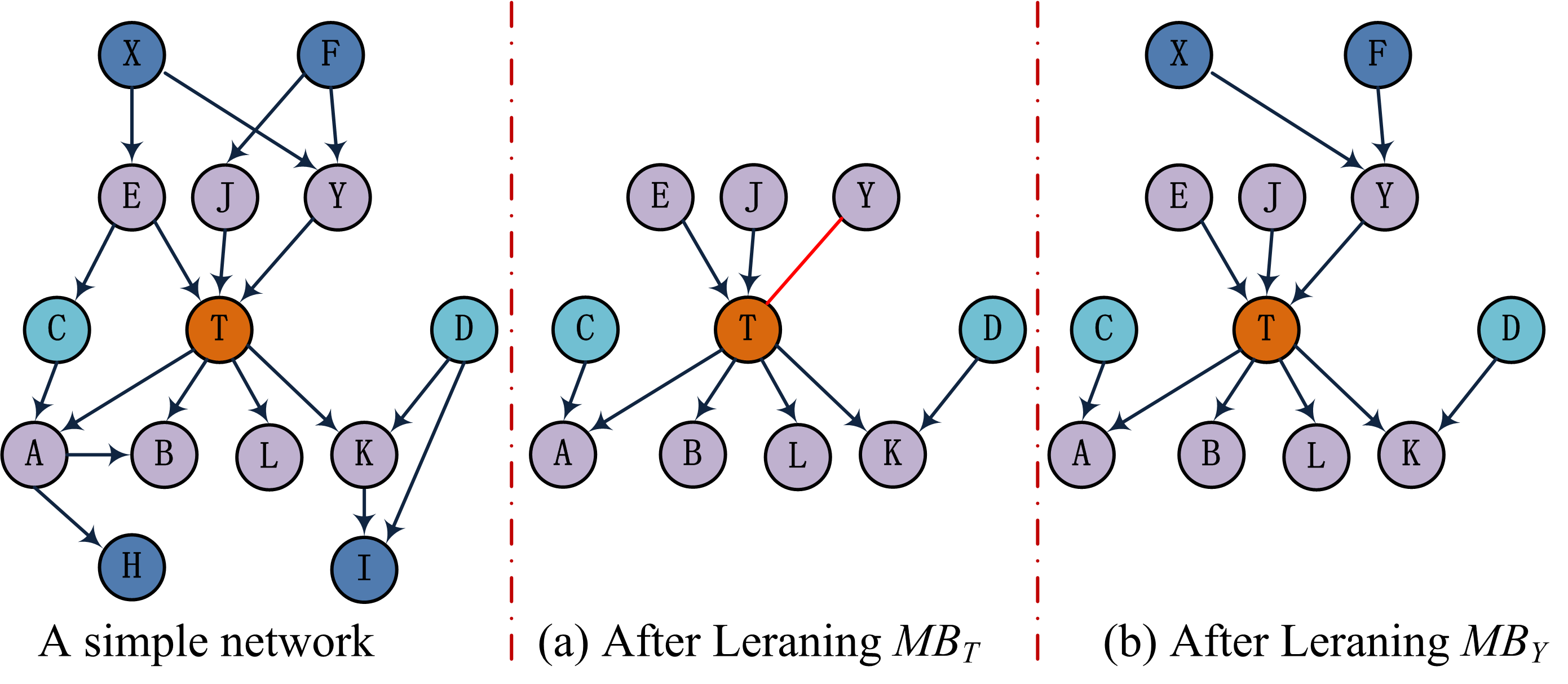}
  \caption{An example of the execution process of ELCS.}
  \label{executexample2}
\end{figure*}
\par  Suppose that we have a dataset for variable set \textbf{\emph{U} } = \{\emph{A}, \emph{B}, \emph{C}, \emph{D}, \emph{E}, \emph{F}, \emph{X}, \emph{H}, \emph{I}, \emph{J}, \emph{K}, \emph{L}, \emph{T}, \emph{Y}\}. The independence relationships between variables can be represented by the BN structure in Fig.\ref{executexample2}. In the following, we regard  \emph{T} as the target variable, and give the execution process of distinguishing  parents from children of \emph{T} using ELCS. We use the $G(X,Y)$ = -1 to represent that  \emph{X} is a parent of  \emph{Y}, $G(X,Y)$ = 1  represents that \emph{X} is adjacent to  \emph{Y}, $G(X,Y)$ = 0  represents that there is no an edge between  \emph{X} and  \emph{Y}.
\par (1) step 1: referring to the simple network, i.e., the left network in Fig. \ref{executexample2}. We first use EMB to distinguish  parents from children of \emph{T}. After learning the MB of \emph{T}, as shown in Fig. \ref{executexample2} (a), the \textbf{\emph{PC}}$_{\emph{T}}$ = \{\emph{A}, \emph{B}, \emph{L}, \emph{K}, \emph{E}, \emph{J}, \emph{Y}\} and  \textbf{\emph{SP}}$_{\emph{T}}$ = \{\emph{C}, \emph{D}\} are obtained. Then, $G(T,A)$ = -1,  $G(T,B)$ = -1,  $G(T,L)$ = -1,  $G(T,K)$ = -1,  $G(E,T)$ = -1,  $G(J,T)$ = -1 and $G(Y,T)$ = 1. The edge direction  between  \emph{Y} and \emph{T} is unsure.
\par (2) step 2: to resolve $G(Y,T)$, as shown in Fig. \ref{executexample2} (b), we need to make a further search. In the following, the MB of \emph{Y} is extracted using EMB, and $G(X,Y)$ = -1,  $G(F,Y)$ = -1. After updating the current local structure using Meek rules, we can learn that \emph{Y} is a parent of \emph{T}, that is, $G(Y,T)$ = -1.

\subsection{Theoretical Analysis}
In the following, we first theoretically analyze the correctness of EMB, then theoretically analyze the correctness of ELCS.
\begin{theorem}[\textbf{Correctness of EMB}]\label{thm3}
Under the faithfulness assumption, and all CI tests are reliable, EMB  finds all and only the MB of a  given  variable.
\end{theorem}

\begin{proof}
\par At step 1, EMB finds all the true PC variables. According to Theorem \ref{thm1}, the variables which are dependent with the target variable \emph{T} will be added to \textbf{\emph{PC}}$_{\emph{T}}$. \textbf{\emph{PC}}$_{\emph{T}}$ contains the true parents and children of \emph{T}, since the true PC variables are always dependent of \emph{T}. In addition, \textbf{\emph{PC}}$_{\emph{T}}$ also contains some descendants of \emph{T} \cite{GaoJ17}. Then, based on the results of  step 1, EMB finds all the true spouses of \emph{T} at step 2. If \emph{Y} is a collider, \emph{X} is regarded as a candidate spouse if there exists a V-structure ($\emph{X}\rightarrow\emph{Y}\leftarrow\emph{T}$) formed by \emph{X}, \emph{Y} and \emph{T}, and \emph{X} will be added to \textbf{\emph{CSP}}$_{\emph{T}}$\{{\emph{Y}}\} (lines 1-16 in Algorithm \ref{FindSpouses}). Owning to an exhaustive search, EMB will not miss any true  spouses of \emph{T}. In the following,  the variable \emph{X} $\in$ \textbf{\emph{SP}}$_{\emph{T}}$\{{\emph{Y}}\} that is a non-parent of \emph{Y} will be removed from \textbf{\emph{SP}}$_{\emph{T}}$\{{\emph{Y}}\} (lines 17-24 in Algorithm \ref{FindSpouses}). According to the Markov condition, the variable  \emph{X} $\in$ \textbf{\emph{SP}}$_{\emph{T}}$\{{\emph{Y}}\} will be removed if it is not a parent of \emph{Y}, since conditioning set \textbf{\emph{SP}}$_{\emph{T}}$\{{\emph{Y}}\} $\cup$ \emph{T} $\cup$  \textbf{\emph{PC}}$_{\emph{T}}$ $\setminus$ \{\emph{X},\emph{Y}\} contains  all true parents of \emph{Y}. Therefore, \textbf{\emph{SP}}$_{\emph{T}}$ contains all the true spouses of \emph{T}.  The learnt \textbf{\emph{PC}}$_{\emph{T}}$ may contain some false PC variables. \emph{T} and each false PC variable are conditionally independent given the spouses of the false PC variable and \textbf{\emph{PC}}$_{\emph{T}}$. At the step 3, the false  PC variables found at step 1 and false  spouses found at step 2 will be removed (lines 3-8 in Algorithm \ref{algorithmEMB}). Then, EMB contains all and only the true  PC variables \textbf{\emph{PC}}$_{\emph{T}}$ and spouses \textbf{\emph{SP}}$_{\emph{T}}$ after executing Algorithm \ref{algorithmEMB}, and \textbf{\emph{PC}}$_{\emph{T}}$ and \textbf{\emph{SP}}$_{\emph{T}}$ together form all and only true  MB variables. At step 4, based on Theorem \ref{thm2} and Lemma \ref{lemma1}, parents and children of \emph{T} are learnt by  EMB  are correct.

\end{proof}

 \par Theorem \ref{thm4}  describes the correctness of the proposed ELCS algorithm. In the following, we will introduce  Theorem \ref{thm4} and its proof in detail.
\begin{theorem}[\textbf{Correctness of ELCS}]\label{thm4}
Under the faithfulness assumption,  and all CI tests are reliable, ELCS  distinguishes all parents from children of a given  variable.
\end{theorem}


\begin{proof}
Under the causal faithfulness assumption, given a target variable, EMB finds all and only the true  MB variables and the true  PC variables of the target variable. The learnt PC set contains all and only the parents and children of the target variable. Based on Definition \ref{defVS}, the children identified by the learnt spouses are correct. Base on Theorem \ref{thm2}, the children identified by the found N-structures are correct. All the parents and children identified by Lemma \ref{lemma1} are correct.  ELCS updates the local causal structures until the parents and children of the target variable have been distinguished. Meek rules \cite{Meek1995Causal} are used to orient other undirected edges between the target variable and the variables that are adjacent to the target variable during learning the local causal structure, and all the edge directions determined by Meek rules are correct. Thus, all the parents and children of a given target variable distinguished by ELCS are correct.
\end{proof}

\subsection{Computational Complexity}
The  computational complexity of ELCS algorithm depends on the steps of discovering  MB. In the following, we will give the computational complexity of EMB and ELCS.
\par \textbf{The computational complexity of EMB:} The computational complexity of EMB depends on its four steps.  Given a target variable \emph{T},  at step 1, the computational cost is dominated by HITON-PC which takes at most $O(|\emph{\textbf{U}}|2^{|\emph{\textbf{PC}}_{\emph{T}}|})$ CI tests to find the PC. Step 2 takes $O(|\textbf{\emph{SP}}_{\emph{T}} |2^{|\textbf{\emph{PC}}_{\emph{T}}|\emph{+}|\textbf{\emph{SP}}_{\emph{T}}|})$   CI tests, step 3 takes $O(|\textbf{\emph{PC}}_{\emph{T}} |2^{|\textbf{\emph{PC}}_{\emph{T}}| \emph{+} |\textbf{\emph{SP}}_{\emph{T}}|})$  CI tests,   step 4 takes $O(2{|\textbf{\emph{PC}}_{\emph{T}}|}^2)$  CI tests. Thus, the complexity of EMB is $O(|\emph{\textbf{U}}|2^{|\emph{\textbf{PC}}_{\emph{T}}|} \emph{+} |\textbf{\emph{SP}}_{\emph{T}} |2^{|\textbf{\emph{PC}}_{\emph{T}}|\emph{+}|\textbf{\emph{SP}}_{\emph{T}}|} + |\textbf{\emph{PC}}_{\emph{T}} |2^{|\textbf{\emph{PC}}_{\emph{T}}|\emph{+}|\textbf{\emph{SP}}_{\emph{T}}|} + 2{|\textbf{\emph{PC}}_{\emph{T}}|}^2)$ = $O(|\emph{\textbf{U}}|2^{|\emph{\textbf{PC}}_{\emph{T}}|})$. Let $|\textbf{\emph{PC}}|$ represent the largest size of the PC sets of all the variables, the complexity of EMB is $O(|\emph{\textbf{U}}|2^{|\emph{\textbf{PC}}|})$.

\par We summarize the computational complexity of EMB and its rivals in Table \ref{OTimeOfMB}. From the table, IAMB is the fastest among all MB learning algorithms. MMMB and HINTON-MB are the slowest two algorithms, since they need to find the PC sets of all target variable's parents and children. The computational complexity of EMB is lower than STMB. In general, most of the BNs have a large number of variables but a small-sized PC set of each variable, so that EMB  is faster than STMB. The computational complexity of BAMB is the same with EMB.

\par \textbf{The computational complexity of ELCS:} In the best case, the complexity of ELCS is $O(|\emph{\textbf{U}}|2^{|\emph{\textbf{PC}}|})$. But in the worst case (e.g. the target variable has all single ancestors), ELCS needs to learn a whole  structure, hence the complexity of ELCS is $O(|\emph{\textbf{U}}|^22^{|\emph{\textbf{PC}}|})$.

\par Table \ref{OTimeOfMLCS} reports the computational complexity of ELCS and its rivals.  Note that $m$ $\ll |\emph{\textbf{U}}|$   is the memory size of L-BFGS \cite{ZhengARX18}. $n$ and $t$ are the number of samples in  data and iterations, respectively. $h$ is the number of neurons in the hidden layer \cite{YuCGY19}. The computational complexity of global causal structure learning algorithms is higher than that of local causal structure learning algorithms, since they learn the whole structure of all variables.  PCD-by-PCD  uses MMPC for PC learning, and CMB uses HITON-MB for MB learning. In the best case,  the complexity of PCD-by-PCD is consistent with that of MMMB, it is $O(|\emph{\textbf{U}}||\textbf{\emph{PC}}|2^{|\textbf{\emph{PC}}|})$. The complexity of CMB is consistent with that of HITON-MB, it is $O(|\emph{\textbf{U}}||\textbf{\emph{PC}}|2^{|\textbf{\emph{PC}}|})$. In the worst case, since MMPC can discover the separating sets while learning PC, the complexity of PCD-by-PCD is consistent with that of  using MMPC to find PCs of all variables, it is $O(|\emph{\textbf{U}}|^{2}2^{|\textbf{\emph{PC}}|})$. The complexity of CMB is  consistent with that of using HITON-MB to find MBs of all variables, it is  $O(|\emph{\textbf{U}}|^{2}|\textbf{\emph{PC}}|2^{|\textbf{\emph{PC}}|})$.  ELCS is the fastest algorithm in both best and worst cases, since the  complexity of ELCS relies on that of EMB which only takes $O(|\emph{\textbf{U}}|2^{|\emph{\textbf{PC}}|})$ CI tests to find the MB.

\tabcolsep 0.15in	
\begin{table*}
\footnotesize
\caption{Computational Complexity of Each Markov Blanket Learning Algorithm}\label{OTimeOfMB}
\centering
\begin{tabular}{ccccccc}
\toprule
Algorithms      &IAMB &MMMB &HITON-MB &STMB &BAMB &EMB\\
\midrule   
Complexity &$O(|\emph{\textbf{U}}|^2)$  &$O(|\emph{\textbf{U}}||\textbf{\emph{PC}}|2^{|\textbf{\emph{PC}}|})$ &$O(|\emph{\textbf{U}}||\textbf{\emph{PC}}|2^{|\textbf{\emph{PC}}|})$  &$O(|\emph{\textbf{U}}|2^{|\emph{\textbf{U}}|})$ & $O(|\emph{\textbf{U}}|2^{|\textbf{\emph{PC}}|})$ &$O(|\emph{\textbf{U}}|2^{|\emph{\textbf{PC}}|})$\\
\bottomrule
\end{tabular}
\end{table*}

\tabcolsep 0.2in	
\begin{table}
\footnotesize
\caption{Computational Complexity of   Local and Global  Causal Structure Learning Algorithms}\label{OTimeOfMLCS}
\centering
\begin{tabular}{ccc}
\toprule
Algorithm      &Worst Case &Best Case\\
\midrule   
PCD-by-PCD &$O(|\emph{\textbf{U}}|^{2}2^{|\textbf{\emph{PC}}|})$ &$O(|\emph{\textbf{U}}||\textbf{\emph{PC}}|2^{|\textbf{\emph{PC}}|})$\\
CMB   &$O(|\emph{\textbf{U}}|^{2}|\textbf{\emph{PC}}|2^{|\textbf{\emph{PC}}|})$&$O(|\emph{\textbf{U}}||\textbf{\emph{PC}}|2^{|\textbf{\emph{PC}}|})$\\
ELCS  &$O(|\emph{\textbf{U}}|^22^{|\emph{\textbf{PC}}|})$ &$O(|\emph{\textbf{U}}|2^{|\emph{\textbf{PC}}|})$\\
MMHC  &$O(|\emph{\textbf{U}}|^22^{|\emph{\textbf{U}}|})$ &$O(|\emph{\textbf{U}}|^22^{|\emph{\textbf{PC}}|})$\\
NOTEARS  &\multicolumn{2}{c}{$O(m^2|\emph{\textbf{U}}|^2 + m^3 + mt|\emph{\textbf{U}}|^2)$} \\
DAG-GNN  &\multicolumn{2}{c}{$O(nht|\emph{\textbf{U}}|^4$)}\\
\bottomrule
\end{tabular}
\end{table}


\section{Experiments}
In this section, we evaluate the performance of the proposed ELCS algorithm, and this section is organized as follows. Section \ref{ExS} gives the experimental settings, Section \ref{ExR} summarizes and discusses the experimental results,  Section \ref{Why} analyses the reason why ELCS is efficient and effective.

\subsection{Experimental Settings}
\label{ExS}
\subsubsection{Datasets}
 We use eight benchmark BNs to evaluate ELCS against its rivals. Each benchmark BN contains two groups of data, one group containing 10 data sets with 5000 data examples,  and the other  one including 10 data sets with 1000 data examples. The number of variables of these BNs ranges from  20 to 801. A brief description of the eight benchmark BNs is listed in  Table \ref{benchmarkBNs}.

\subsubsection{Comparison Methods}
We compare our approach ELCS with three state-of-the-art global causal structure learning algorithms, including MMHC \cite{TsamardinosBA06}, NOTEARS \cite{ZhengARX18} and DAG-GNN \cite{YuCGY19}, and two local causal structure learning algorithms, including PCD-by-PCD \cite{YinZWHZG08} and CMB \cite{GaoJ15}. In addition, we also compare ELCS with ELCS-II.

\subsubsection{Implementation Details}
PCD-by-PCD and CMB algorithms are implemented by ourselves in MATLAB (https://github.com/kuiy/CausalLearner). For MMHC, we use the implementation in the software tool of Causal Explorer \cite{AliferisTSB03}.  For NOTEARS and DAG-GNN, we use the  source codes provided by the authors. In the experiments, $G^2$-test with the significance level of 0.01 is utilized to measure the conditional independence between variables. All experimental results are conducted on Windows 10 with Intel(R) i7-8700, 3.19 GHz CPU, and 16GB memory.

\subsubsection{Evaluation Metrics}
In the experiments, we evaluate the performance using the following metrics.
\begin{itemize}
\item \emph{ArrP}: The number of true directed edges in the output (i.e., the variables in the output belonging to the true parents and children of a target variable in a test DAG) divided by the number of edges in the output of an algorithm.
\item \emph{ArrR}:  The number of true directed edges in the output divided by the number of true directed edges (i.e., the number of parents and children of a target variable in a test DAG).
\item \emph{SHD}:  SHD is the number of total error edges, which contains undirected edges, reverse edges, missing edges and extra edges. The smaller value of SHD is better.
\item \emph{FDR}:  The number of false edges in the output divided by the number of edges in the output of an algorithm.
\item \emph{Efficiency}: The number of CI tests and the running time (in seconds) are utilized to measure the efficiency.
\end{itemize}

In the following Tables, the results are reported in the format of $A \pm B$, where $A$ denotes the average results, and $B$ represents the standard deviation. The best results in each setting have been marked in bold. "-" means that the output of the corresponding BN cannot be generated  in two days by the algorithm. Note that NOTEARS and DAG-GNN do not perform CI tests.
 \subsection{Experimental Results of ELCS and Its Rivals}
 \label{ExR}

\tabcolsep 0.05in
\newcommand{\tabincell}[2]{\begin{tabular}{@{}#1@{}}#2\end{tabular}}
\begin{table}
\footnotesize
\centering
\caption{Summary of Benchmark BNs}
\label{benchmarkBNs}       
\begin{tabular}{cccccc}
\toprule
\tabincell{c}{Network}  & \tabincell{c}{Num. \\ Vars }   & \tabincell{c}{Num. \\ Edges} &\tabincell{c}{Max In/Out\\ Degree} &\tabincell{c}{Min/Max\\ $|$PCset$|$}  &\tabincell{c}{Domain\\ Range}\\
\midrule
Alarm      &  37 & 46 & 4 / 5 & 1 / 6 & 2 - 4\\
Insurance  &  27 & 52 & 3 / 7 & 1 / 9 & 2 - 5\\
Child      &  20 & 25 & 2 / 7 & 1 / 8 & 2 - 6\\
Alarm10      &  370 & 570 & 4 / 7 & 1 / 9 & 2 - 4\\
Insurance10  &  270 & 556 & 5 / 8 & 1 / 11 & 2 - 5\\
Child10      &  200 & 257 & 2 / 7 & 1 / 8 & 2 - 6\\
Pigs         &  441 & 592 & 2 / 39 & 1 / 41 & 3 - 3\\
Gene        &  801 & 972  & 4 / 10 & 0 / 11 & 3 - 5\\
\bottomrule
\end{tabular}
\end{table}

\par We compare ELCS with  MMHC, NOTEARS, DAG-GNN, PCD-by-PCD, CMB and ELCS-II on the eight BNs as shown in Table \ref{benchmarkBNs}. The average results of ArrP, ArrR, SHD, FDR, CI tests and running time of each algorithm are reported in Tables \ref{Local5000}-\ref{Local1000}. Table \ref{Local5000} summarizes the experimental results on the eight  BNs with 5,000 data examples, and Table \ref{Local1000} reports the experimental results on the eight  BNs with 1,000 data examples.  From the experimental results, we have the following observations.

\tabcolsep 0.03in
\begin{table}
\tiny
\caption{ Comparison of ELCS with State-of-the-art Causal Structure Learning Algorithms on Eight Benchmark BNs (size=5,000)}\label{Local5000}
\centering
\begin{tabular}{c||c||cccccc}\hline
Network & Algorithm  & ArrP($\uparrow$)   & ArrR($\uparrow$) &SHD($\downarrow$)   &FDR($\downarrow$) &CI Tests($\downarrow$) &Time($\downarrow$) \\\hline\hline
\multirow{7}{*}{Alarm}
&MMHC 	&0.19$\pm$0.02	&0.08$\pm$0.02		&4.58$\pm$0.02	&0.60$\pm$0.01		&13860$\pm$4971	&7.03$\pm$2.74 \\
&NOTEARS 	&0.61$\pm$0.01	&0.74$\pm$0.04		&2.84$\pm$0.17	&0.48$\pm$0.02		&-	&541.95$\pm$27.18 \\
&DAG-GNN	&0.66$\pm$0.03	&0.54$\pm$0.04		&2.02$\pm$0.17	&0.36$\pm$0.05		&-	&1.1e3$\pm$1.1e2 \\
&PCD-by-PCD 	&0.77$\pm$0.03	&0.64$\pm$0.05		&0.94$\pm$0.15	&0.27$\pm$0.04		&2110$\pm$92	&0.81$\pm$0.04 \\
&CMB 	        &0.77$\pm$0.05	&0.72$\pm$0.06		&0.76$\pm$0.14	&0.22$\pm$0.04		&2111$\pm$207	&0.69$\pm$0.07\\
&ELCS 	        &\textbf{0.86$\pm$0.01}	&\textbf{0.81$\pm$0.01}		&\textbf{0.44$\pm$0.06}	&\textbf{0.07$\pm$0.02}		&648$\pm$55	    &0.20$\pm$0.02 \\
&ELCS-II 	    &\textbf{0.86$\pm$0.01}	&\textbf{0.81$\pm$0.01}		&\textbf{0.44$\pm$0.06}	&\textbf{0.07$\pm$0.02}		&\textbf{607$\pm$52}	    &\textbf{0.19$\pm$0.02} \\\hline
\multirow{7}{*}{Insurance}
&MMHC 	&0.21$\pm$0.02	&0.03$\pm$0.02		&5.87$\pm$0.17	&0.67$\pm$0.02		&2603$\pm$271	&1.18$\pm$0.12 \\
&NOTEARS 	&0.46$\pm$0.01	&0.24$\pm$0.01		&4.64$\pm$0.12	&0.72$\pm$0.02		&-	&420.89$\pm$22.92 \\
&DAG-GNN	&0.54$\pm$0.03	&0.21$\pm$0.01		&4.35$\pm$0.25	&0.67$\pm$0.05		&-	&518.84$\pm$37.24 \\
&PCD-by-PCD 	&0.68$\pm$0.02	&0.45$\pm$0.02		&2.07$\pm$0.09	&0.34$\pm$0.03		&3038$\pm$300	&1.48$\pm$0.16 \\
&CMB 	        &0.70$\pm$0.04	&0.54$\pm$0.04		&2.31$\pm$0.25	&0.37$\pm$0.05		&11553$\pm$4827	&5.44$\pm$2.29 \\
&ELCS 	        &\textbf{0.85$\pm$0.04}	&\textbf{0.69$\pm$0.04}		&\textbf{1.61$\pm$0.06}	&\textbf{0.18$\pm$0.05}		&1686$\pm$276	&\textbf{0.75$\pm$0.12} \\
&ELCS-II	    &\textbf{0.85$\pm$0.04}	&\textbf{0.69$\pm$0.04}		&\textbf{1.61$\pm$0.06}	&\textbf{0.18$\pm$0.05}		&\textbf{1637$\pm$275}	&\textbf{0.75$\pm$0.12} \\\hline
\multirow{7}{*}{Child}
&MMHC 	&0.22$\pm$0.03	&0.19$\pm$0.07		&3.63$\pm$0.25	&0.48$\pm$0.03		&8600$\pm$632	&5.32$\pm$0.46 \\
&NOTEARS 	&0.52$\pm$0.02	&0.39$\pm$0.03		&2.99$\pm$0.17	&0.70$\pm$0.03		&-	&140.74$\pm$36.59 \\
&DAG-GNN	&0.50$\pm$0.04	&0.29$\pm$0.06		&2.08$\pm$0.10	&0.44$\pm$0.06		&-	&384.73$\pm$30.76 \\
&PCD-by-PCD 	&0.71$\pm$0.02	&0.59$\pm$0.04		&0.86$\pm$0.09	&0.26$\pm$0.04		&2432$\pm$78	&1.24$\pm$0.04 \\
&CMB 	        &\textbf{0.82$\pm$0.05}	&\textbf{0.75$\pm$0.08}  	&\textbf{0.72$\pm$0.18}	&0.25$\pm$0.08		&9424$\pm$4106	&4.58$\pm$1.96\\
&ELCS 	        &0.71$\pm$0.12	&0.61$\pm$0.16		&1.08$\pm$0.36	&\textbf{0.09$\pm$0.08}		&2093$\pm$287	&\textbf{0.93$\pm$0.10} \\
&ELCS-II 	    &0.71$\pm$0.12	&0.61$\pm$0.16		&1.08$\pm$0.36	&\textbf{0.09$\pm$0.08}		&\textbf{2087$\pm$287}	&\textbf{0.93$\pm$0.10} \\\hline
\multirow{7}{*}{Alarm10}
&MMHC 	&0.19+0.00	&0.02+0.00		&5.63+0.05	&0.63+0.00		&9.7e7+8.9e6	&4.6e4+4.8e3 \\
&NOTEARS 	&0.73$\pm$0.01	&0.50$\pm$0.01		&2.27$\pm$0.04	&0.28$\pm$0.01		&-	&1.6e4$\pm$1.8e3 \\
&DAG-GNN	&-	&-		&-	&-		&-	&- \\
&PCD-by-PCD 	&0.73$\pm$0.01	&0.54$\pm$0.01		&1.48$\pm$0.03	&0.21$\pm$0.01		&25795$\pm$1770	&8.18$\pm$0.57\\
&CMB 	        &0.72$\pm$0.01	&0.58$\pm$0.01		&1.57$\pm$0.04	&0.34$\pm$0.01		&14011$\pm$565	&3.69$\pm$0.15 \\
&ELCS 	        &\textbf{0.83$\pm$0.01}	&\textbf{0.68$\pm$0.02}		&\textbf{1.26$\pm$0.07}	&\textbf{0.14$\pm$0.02}		&\textbf{6893$\pm$483}	&\textbf{1.77$\pm$0.12}\\
&ELCS-II 	    &\textbf{0.83$\pm$0.01}	&\textbf{0.68$\pm$0.02}		&\textbf{1.26$\pm$0.07}	&\textbf{0.14$\pm$0.02}		&6916$\pm$480	&\textbf{1.77$\pm$0.12}\\\hline
\multirow{7}{*}{Insurance10}
&MMHC 	&0.22$\pm$0.00	&0.00$\pm$0.00		&6.72$\pm$0.04	&0.70$\pm$0.00		&1.9e5$\pm$2.0e4	&81.66$\pm$10.39 \\
&NOTEARS 	&0.30$\pm$0.01	&0.20$\pm$0.00		&8.67$\pm$0.44	&0.85$\pm$0.01		&-	&1.7e4$\pm$1.5e4 \\
&DAG-GNN	&-	&-		&-	&-		&-	&- \\
&PCD-by-PCD 	&0.68$\pm$0.01	&0.46$\pm$0.01		&2.10$\pm$0.05	&0.41$\pm$0.01		&\textbf{9581$\pm$224}	&4.38$\pm$0.13\\
&CMB 	        &0.64$\pm$0.01	&0.49$\pm$0.01		&2.58$\pm$0.06	&0.48$\pm$0.02		&39932$\pm$3898	&16.04$\pm$1.54 \\
&ELCS 	        &\textbf{0.80$\pm$0.02}	&\textbf{0.67$\pm$0.02}		&\textbf{1.75$\pm$0.11}	&\textbf{0.23$\pm$0.01}		&10809$\pm$1528	&3.92$\pm$0.55\\
&ELCS-II	    &\textbf{0.80$\pm$0.02}	&\textbf{0.67$\pm$0.02}		&\textbf{1.75$\pm$0.11}	&\textbf{0.23$\pm$0.01}		&10605$\pm$1499	&\textbf{3.91$\pm$0.55}\\\hline
\multirow{7}{*}{Child10}
&MMHC 	        &0.15$\pm$0.01	&0.03$\pm$0.01		&5.29$\pm$0.09	&0.58$\pm$0.01		&7.9e5$\pm$2.7e5	&439.79$\pm$169.28 \\
&NOTEARS 	    &0.61$\pm$0.01	&0.74$\pm$0.04		&2.84$\pm$0.17	&0.48$\pm$0.02		&-	&541.95$\pm$27.18 \\
&DAG-GNN	&-	&-		&-	&-		&-	&- \\
&PCD-by-PCD 	&0.77$\pm$0.01	&0.68$\pm$0.02		&0.82$\pm$0.04	&0.22$\pm$0.01		&\textbf{11341$\pm$1470}	&4.44$\pm$0.57\\
&CMB 	        &0.75$\pm$0.02	&0.68$\pm$0.02		&1.03$\pm$0.07	&0.31$\pm$0.02		&22861$\pm$2648	&8.11$\pm$0.95 \\
&ELCS 	        &\textbf{0.83$\pm$0.05}	&\textbf{0.76$\pm$0.07}		&\textbf{0.73$\pm$0.20}	&\textbf{0.14$\pm$0.03}		&13129$\pm$2613	&\textbf{3.98$\pm$0.79}\\
&ELCS-II	        &\textbf{0.83$\pm$0.05}	&\textbf{0.76$\pm$0.07}		&\textbf{0.73$\pm$0.20}	&\textbf{0.14$\pm$0.03}		&13104$\pm$2605	&\textbf{3.98$\pm$0.79}\\\hline
\multirow{7}{*}{Pigs}
&MMHC 	&0.26$\pm$0.00	&0.00$\pm$0.00		&6.85$\pm$0.07	&1.00$\pm$0.00		&4.3e5$\pm$1.4e4	&207.96$\pm$6.88 \\
&NOTEARS 	&0.43$\pm$0.00	&0.26$\pm$0.00	&2.77$\pm$0.03	&0.77$\pm$0.00	&-	&3.1e4$\pm$1.2e3 \\
&DAG-GNN	&-	&-		&-	&-		&-	&- \\
&PCD-by-PCD 	&-	&-		&-	&-		&-	&- \\
&CMB 	        &-	&-		&-	&-		&-	&- \\
&ELCS 	        &\textbf{0.91$\pm$0.00}	&\textbf{0.99$\pm$0.00}		&\textbf{0.42$\pm$0.02}	&\textbf{0.15$\pm$0.01}		&13374$\pm$8660	&8.91$\pm$6.84 \\
&ELCS-II 	        &\textbf{0.91$\pm$0.00}	&\textbf{0.99$\pm$0.00}		&\textbf{0.42$\pm$0.02}	&\textbf{0.15$\pm$0.01}		&\textbf{11467$\pm$5659} &\textbf{8.38$\pm$5.12} \\\hline
\multirow{7}{*}{Gene}
&MMHC 	&-	&-		&-	&-		&-	&- \\
&NOTEARS 	&-	&-		&-	&-		&-	&- \\
&DAG-GNN	&-	&-		&-	&-		&-	&- \\
&PCD-by-PCD 	&-	&-		&-	&-		&-	&- \\
&CMB 	        &-	&-		&-	&-		&-	&- \\
&ELCS 	        &\textbf{0.76$\pm$0.01}	&\textbf{0.79$\pm$0.01}		&\textbf{0.79$\pm$0.03}	&\textbf{0.32$\pm$0.01}		&36950$\pm$7876	&11.03$\pm$2.35 \\
&ELCS-II 	        &\textbf{0.76$\pm$0.01}	&\textbf{0.79$\pm$0.01}		&\textbf{0.79$\pm$0.03}	&\textbf{0.32$\pm$0.01}		&\textbf{36051$\pm$7696}	&\textbf{11.02$\pm$2.08} \\\hline
\end{tabular}
\end{table}

\tabcolsep 0.03in
\begin{table}
\tiny
\caption{ Comparison of ELCS with State-of-the-art Causal Structure Learning Algorithms on Eight Benchmark BNs (size=1,000)}\label{Local1000}
\centering
\begin{tabular}{c||c||cccccc}
\hline
Network & Algorithm  & ArrP($\uparrow$)   & ArrR($\uparrow$) &SHD($\downarrow$)   &FDR($\downarrow$) &CI Tests($\downarrow$) &Time($\downarrow$)\\\hline\hline
\multirow{7}{*}{Alarm}
&MMHC 	        &0.22$\pm$0.03	&0.12$\pm$0.04		&4.32$\pm$0.20	&0.57$\pm$0.03		&8884$\pm$4471	&1.64$\pm$0.81 \\
&NOTEARS 	    &0.59$\pm$0.03	&\textbf{0.72$\pm$0.06}		&3.14$\pm$0.21	&0.51$\pm$0.04		&-	&232.28$\pm$22.13 \\
&DAG-GNN	    &0.48$\pm$0.04	&0.26$\pm$0.06		&2.01$\pm$0.10	&0.17$\pm$0.03		&-	&241.19$\pm$23.82 \\
&PCD-by-PCD 	&0.66$\pm$0.07	&0.49$\pm$0.05		&1.33$\pm$0.10	&0.22$\pm$0.04		&1737$\pm$265	&0.39$\pm$0.05 \\
&CMB 	        &0.67$\pm$0.06	&0.52$\pm$0.06		&1.32$\pm$0.13	&0.34$\pm$0.07		&3171$\pm$410	&0.50$\pm$0.07 \\
&ELCS 	        &\textbf{0.72$\pm$0.07}	&0.61$\pm$0.08		&\textbf{1.06$\pm$0.14}	&\textbf{0.11$\pm$0.04}		&901$\pm$172	&\textbf{0.13$\pm$0.03}\\
&ELCS-II 	    &\textbf{0.72$\pm$0.07}	&0.61$\pm$0.08		&\textbf{1.06$\pm$0.14}	&\textbf{0.11$\pm$0.04}		&\textbf{861$\pm$162}	&\textbf{0.13$\pm$0.03}\\\hline
\multirow{7}{*}{Insurance}
&MMHC 	        &0.22$\pm$0.02	&0.04$\pm$0.02		&5.72$\pm$0.18	&0.65$\pm$0.03		&2110$\pm$293	&0.44$\pm$0.05 \\
&NOTEARS 	    &0.43$\pm$0.02	&0.24$\pm$0.01		&4.90$\pm$0.12	&0.75$\pm$0.03		&-	&220.97$\pm$44.51 \\
&DAG-GNN	    &0.49$\pm$0.06	&0.15$\pm$0.05		&3.78$\pm$0.19	&0.39$\pm$0.13		&-	&151.75$\pm$18.26 \\
&PCD-by-PCD 	&0.68$\pm$0.04	&0.40$\pm$0.04		&\textbf{2.43$\pm$0.15}	&\textbf{0.30$\pm$0.06}		&1370$\pm$104	&0.33$\pm$0.02 \\
&CMB 	        &\textbf{0.69$\pm$0.06}	&\textbf{0.46$\pm$0.05}		&2.55$\pm$0.22	&0.38$\pm$0.08		&4457$\pm$1196	&0.76$\pm$0.20 \\
&ELCS 	        &\textbf{0.69$\pm$0.12}	&0.44$\pm$0.15		&2.49$\pm$0.40	&0.37$\pm$0.19		&1188$\pm$566	&\textbf{0.17$\pm$0.08}\\
&ELCS-II	    &\textbf{0.69$\pm$0.12}	&0.44$\pm$0.15		&2.49$\pm$0.40	&0.37$\pm$0.19		&\textbf{1106$\pm$513}	&\textbf{0.17$\pm$0.08}\\\hline
\multirow{7}{*}{Child}
&MMHC 	        &0.24$\pm$0.02	&0.18$\pm$0.04		&3.41$\pm$0.14	&0.45$\pm$0.03		&4583$\pm$898	&0.90$\pm$0.19 \\
&NOTEARS 	    &0.49$\pm$0.02	&0.37$\pm$0.05		&3.31$\pm$0.23	&0.72$\pm$0.04		&-	&66.32$\pm$25.78 \\
&DAG-GNN	    &0.34$\pm$0.05	&0.15$\pm$0.03		&2.20$\pm$0.05	&0.29$\pm$0.09		&-	&87.70$\pm$8.96 \\
&PCD-by-PCD 	&0.52$\pm$0.05	&0.34$\pm$0.06		&1.61$\pm$0.12	&0.33$\pm$0.08		&\textbf{2085$\pm$183}	&0.39$\pm$0.02 \\
&CMB 	        &0.74$\pm$0.09	&0.59$\pm$0.10		&1.27$\pm$0.29	&0.33$\pm$0.12		&4991$\pm$1145	&0.65$\pm$0.14 \\
&ELCS 	        &\textbf{0.82$\pm$0.05}	&\textbf{0.69$\pm$0.06}		&\textbf{1.01$\pm$0.18}	&\textbf{0.21$\pm$0.06}		&2882$\pm$815	&0.34$\pm$0.10\\
&ELCS-II 	    &\textbf{0.82$\pm$0.05}	&\textbf{0.69$\pm$0.06}		&\textbf{1.01$\pm$0.18}	&\textbf{0.21$\pm$0.06}		&2592$\pm$723	&\textbf{0.33$\pm$0.10}\\\hline
\multirow{7}{*}{Alarm10}
&MMHC 	        &0.19$\pm$0.00	&0.03$\pm$0.01		&5.69$\pm$0.07	&0.63$\pm$0.00		&3.9e6$\pm$3.5e5	&700.63$\pm$55.38 \\
&NOTEARS 	    &0.39$\pm$0.01	&0.47$\pm$0.01		&9.27$\pm$0.49	&0.69$\pm$0.02		&-	&1.6e4$\pm$1.2e3 \\
&DAG-GNN	    &-	&-		&-	&-		&-	&- \\
&PCD-by-PCD 	&0.66$\pm$0.01	&0.44$\pm$0.02		&1.74$\pm$0.06	&\textbf{0.20$\pm$0.02}		&26572$\pm$3414	&4.87$\pm$0.63 \\
&CMB 	        &0.68$\pm$0.01	&0.48$\pm$0.02		&1.90$\pm$0.06	&0.39$\pm$0.02		&10827$\pm$643	&1.51$\pm$0.08 \\
&ELCS 	        &\textbf{0.75$\pm$0.02}	&\textbf{0.53$\pm$0.02}		&\textbf{1.72$\pm$0.06}	&\textbf{0.20$\pm$0.03}		&8800$\pm$1218	&\textbf{1.18$\pm$0.16}\\
&ELCS-II 	    &\textbf{0.75$\pm$0.02}	&\textbf{0.53$\pm$0.02}		&\textbf{1.72$\pm$0.06}	&\textbf{0.20$\pm$0.03}		&\textbf{8745$\pm$1209}	&\textbf{1.18$\pm$0.16}\\\hline
\multirow{7}{*}{Insurance10}
&MMHC 	        &0.24$\pm$0.01	&0.05$\pm$0.01		&6.57$\pm$0.05	&0.63$\pm$0.01		&9.6e5$\pm$1.2e5	&236.11$\pm$25.93 \\
&NOTEARS 	    &0.20$\pm$0.02	&0.20$\pm$0.00		&14.11$\pm$0.92	&0.91$\pm$0.01		&-	&9.0e3$\pm$1.1e3 \\
&DAG-GNN	    &-	&-		&-	&-		&-	&- \\
&PCD-by-PCD 	&\textbf{0.63$\pm$0.01}	&0.37$\pm$0.01		&\textbf{2.66$\pm$0.05}	&0.46$\pm$0.02		&8461$\pm$1809	&1.60$\pm$0.35 \\
&CMB 	        &0.62$\pm$0.01	&\textbf{0.45$\pm$0.01}		&2.95$\pm$0.05	&\textbf{0.45$\pm$0.02}		&20895$\pm$2158	&3.23$\pm$0.33 \\
&ELCS 	        &0.50$\pm$0.01	&0.26$\pm$0.00		&3.18$\pm$0.04	&0.65$\pm$0.00		&4333$\pm$1736	&0.60$\pm$0.25\\
&ELCS-II 	        &0.50$\pm$0.01	&0.26$\pm$0.00		&3.18$\pm$0.04	&0.65$\pm$0.00		&\textbf{3966$\pm$1423}	&\textbf{0.59$\pm$0.25}\\\hline
\multirow{7}{*}{Child10}
&MMHC 	        &0.22$\pm$0.01	&0.19$\pm$0.02		&4.37$\pm$0.11	&0.48$\pm$0.01		&7.7e6$\pm$2.0e6	&1.5e3$\pm$3.9e2 \\
&NOTEARS 	    &0.49$\pm$0.01	&0.34$\pm$0.02		&2.99$\pm$0.10	&0.65$\pm$0.02		&-	&3.3e3$\pm$1.9e2 \\
&DAG-GNN	    &-	&-		&-	&-		&-	&- \\
&PCD-by-PCD 	&0.55$\pm$0.02	&0.36$\pm$0.03		&1.69$\pm$0.06	&0.38$\pm$0.03		&15698$\pm$3819	&2.82$\pm$0.71 \\
&CMB 	        &\textbf{0.71$\pm$0.04}	&\textbf{0.59$\pm$0.03}		&1.58$\pm$0.12	&\textbf{0.35$\pm$0.03}		&26986$\pm$3942	&3.71$\pm$0.54 \\
&ELCS 	        &0.67$\pm$0.03	&0.55$\pm$0.02		&\textbf{1.56$\pm$0.07}	&0.36$\pm$0.02		&5074$\pm$658	&0.67$\pm$0.09\\
&ELCS-II	    &0.67$\pm$0.03	&0.55$\pm$0.02		&\textbf{1.56$\pm$0.07}	&0.36$\pm$0.02		&\textbf{4889$\pm$631}	&\textbf{0.66$\pm$0.09}\\\hline
\multirow{7}{*}{Pigs}
&MMHC 	        &0.26$\pm$0.00	&0.00$\pm$0.00		&6.72$\pm$0.02	&1.00$\pm$0.00		&4.6e5$\pm$9.9e3	&90.35$\pm$2.90 \\
&NOTEARS 	    &0.42$\pm$0.00	&0.22$\pm$0.01		&2.85$\pm$0.03	&0.80$\pm$0.01		&-	&2.4e4$\pm$1.7e3 \\
&DAG-GNN	    &-	&-		&-	&-		&-	&- \\
&PCD-by-PCD 	&-	&-		&-	&-		&-	&- \\
&CMB 	        &-	&-		&-	&-		&-	&- \\
&ELCS 	        &\textbf{0.91$\pm$0.01}	&\textbf{0.99$\pm$0.00}		&\textbf{0.40$\pm$0.03}	&\textbf{0.15$\pm$0.01}		&11793$\pm$3279	&\textbf{0.84$\pm$0.13}\\
&ELCS-II 	        &\textbf{0.91$\pm$0.01}	&\textbf{0.99$\pm$0.00}		&\textbf{0.40$\pm$0.03}	&\textbf{0.15$\pm$0.01}		&\textbf{11685$\pm$3272}	&\textbf{0.84$\pm$0.13}\\\hline
\multirow{7}{*}{Gene}
&MMHC 	&-	&-		&-	&-		&-	&- \\
&NOTEARS 	&-	&-		&-	&-		&-	&- \\
&DAG-GNN	&-	&-		&-	&-		&-	&- \\
&PCD-by-PCD 	&-	&-		&-	&-		&-	&- \\
&CMB 	        &-	&-		&-	&-		&-	&- \\
&ELCS 	        &\textbf{0.77$\pm$0.00}	&\textbf{0.78$\pm$0.01}		&\textbf{0.78$\pm$0.02}	&\textbf{0.31$\pm$0.01}		&31753$\pm$3432	&4.37$\pm$0.47 \\
&ELCS-II 	        &\textbf{0.77$\pm$0.00}	&\textbf{0.78$\pm$0.01}		&\textbf{0.78$\pm$0.02}	&\textbf{0.31$\pm$0.01}		&\textbf{31430$\pm$3406}	&\textbf{4.36$\pm$0.47} \\\hline
\end{tabular}
\end{table}



\par \textbf{ELCS \emph{versus} MMHC.} Regardless of the number of samples (5000 or 1000), ELCS is significantly better than MMHC. On the ArrP and ArrR metrics, ELCS is superior to MMHC, which means that ELCS finds more true causal edges and less false casual edges. In addition, on the SHD metric, the value of SHD of ELCS is significantly lower than that of MMHC. On the FDR metric, ELCS performs better than MMHC. Furthermore, ELCS always uses less CI tests than MMHC.  To learn the local causal structure of a target variable,  MMHC needs to learn the whole DAG over all variables in a dataset, hence MMHC performs much more CI tests than ELCS. Thus, we can conclude that ELCS is more efficient and effective than  MMHC.
\par \textbf{ELCS \emph{versus} NOTEARS and DAG-GNN.} NOTEARS and DAG-GNN are  global causal learning algorithms, they need to learn the global structure over all variables, and then obtain the parents and children of a given variable. ELCS achieves better performance than NOTEARS and DAG-GNN using both 5,000 data samples and 1,000 data samples, especially using 5,000 data samples.  On the ArrP, ArrR, SHD and FDR metrics, ELCS  is significantly better than NOTEARS and DAG-GNN. The values of ELCS on ArrP and ArrR metrics  are higher than that of NOTEARS and DAG-GNN, and lower on the SHD and FDR metrics. Since NOTEARS and DAG-GNN  adopt a continuous optimization strategy to obtain the DAG from observational data, and the experimental results are susceptible to the influence of parameters. Additionally,  NOTEARS and DAG-GNN need to spend much time in learning the DAG, since they obtain the optimal solution by means of a large number of iterations. In a word,  ELCS is superior to NOTEARS and DAG-GNN.
\par \textbf{ELCS \emph{versus} PCD-by-PCD and CMB.} Both PCD-by-PCD and CMB are  local causal learning algorithms.  Using 5,000 data samples, ELCS performs better  than PCD-by-PCD and CMB. Except on Child, ELCS achieves highest ArrP and  ArrR values, and lowest  SHD and FDR values on the other BNs. In addition, ELCS uses less CI tests than PCD-by-PCD and CMB on most of BNs.  Using 1,000 data samples, ELCS is superior to PCD-by-PCD and CMB on Alarm, Child, Alarm10, Pigs and Gene.  ELCS is better than CMB  and worse than PCD-by-PCD on Insurance on the  ArrP, ArrR,  SHD and FDR metrics, while ELCS has advantages in terms of CI tests and running time.  On Insurance10,  ELCS is worse than PCD-by-PCD and CMB on the  ArrP, ArrR,  SHD and FDR metrics. The reason may be that EMB learns inaccurate MBs  on  small  data samples. ELCS is better than PCD-by-PCD  and little worse than CMB on Child10 on the  ArrP, ArrR,  SHD and FDR metrics. Generally, ELCS performs better than PCD-by-PCD and CMB.

\par \textbf{ELCS \emph{versus} ELCS-II.} ELCS-II is superior to ELCS. ELCS-II further improves the efficiency of EMB while maintaining the same performance as measured by the ArrP, ArrR,  SHD and FDR metrics, which indicates the efficiency of ELCS-II.

In summary, it can be seen from Tables \ref{Local5000}-\ref{Local1000}, ELCS is significantly better than MMHC, NOTEARS and DAG-GNN.   Additionally, ELCS outperforms  PCD-by-PCD and CMB on the ArrP, ArrR, SHD, FDR metrics. Specifically, compared with PCD-by-PCD and CMB, ELCS not only achieves higher ArrP and ArrR values,  but also achieves lower SHD and FDR  values. Furthermore,  ELCS is the fastest  algorithm  among all structure learning algorithms. ELCS is significantly better  than MMHC and NOTEARS  in terms of running time. MMHC, NOTEARS and DAG-GNN are global causal learning algorithms, they need to find the global structure of a BN. In particular, ELCS is 10 times faster  than MMHC and 1000 times faster than NOTEARS and DAG-GNN on average. Additionally, ELCS is  also superior to PCD-by-PCD and CMB in terms of  running times. ELCS is 2 times faster than PCD-by-PCD and 3 times faster than  CMB on average. Specifically, MMHC, NOTEARS, PCD-by-PCD and CMB fail to generate the output on several BNs, whereas ELCS can be successful applied in learning the local causal structure of each variable within two days. But beyond that, ELCS uses the smallest number of CI tests.  Overall, ELCS  is superior to its rivals  in both efficiency and accuracy.

\tabcolsep 0.03in
\begin{table}
\tiny
\caption{Comparison of ELCS with ``ECLS w/o N" on Eight Benchmark BNs (size=5,000)}\label{caseNLocal5000}
\centering
\begin{tabular}{c||c||cccccc}\hline
Network & Algorithm  & ArrP($\uparrow$)   & ArrR($\uparrow$) &SHD($\downarrow$)   &FDR($\downarrow$) &CI Tests($\downarrow$) &Time($\downarrow$) \\\hline\hline
\multirow{2}{*}{Alarm}
&ELCS 	        &0.86$\pm$0.01	&\textbf{0.81$\pm$0.01}		&0.44$\pm$0.06	&0.07$\pm$0.02		&\textbf{648$\pm$55}	    &\textbf{0.20$\pm$0.02} \\
&ELCS w/o N 	        &\textbf{0.87$\pm$0.01} &\textbf{0.81$\pm$0.01} &\textbf{0.40$\pm$0.06} &\textbf{0.05$\pm$0.02} &701$\pm$99 &0.22$\pm$0.03 \\\hline
\multirow{2}{*}{Insurance}
&ELCS         &\textbf{0.85$\pm$0.04}	&\textbf{0.69$\pm$0.04}		&\textbf{1.61$\pm$0.06}	&\textbf{0.18$\pm$0.05}		&\textbf{1686$\pm$276}	&\textbf{0.75$\pm$0.12} \\
&ELCS w/o N 	&\textbf{0.85$\pm$0.04}	&0.68$\pm$0.03	&1.61$\pm$0.14	&\textbf{0.18$\pm$0.05}		&1924$\pm$257	   &0.84$\pm$0.11 \\\hline
\multirow{2}{*}{Child}
&ELCS 	        &\textbf{0.71$\pm$0.12}	&\textbf{0.61$\pm$0.16}		&\textbf{1.08$\pm$0.36}	&\textbf{0.09$\pm$0.08}		&\textbf{2093$\pm$287}	&\textbf{0.93$\pm$0.10} \\
&ELCS w/o N     &0.70$\pm$0.13	&0.59$\pm$0.16	&1.13$\pm$0.36	&0.09$\pm$0.09		&2143$\pm$277	   &0.94$\pm$0.09 \\\hline
\multirow{2}{*}{Alarm10}
&ELCS 	        &0.83$\pm$0.01	&\textbf{0.68$\pm$0.02}		&1.26$\pm$0.07	&\textbf{0.14$\pm$0.02}		&\textbf{6893$\pm$483}	&\textbf{1.77$\pm$0.12}\\
&ELCS w/o N 	&\textbf{0.84$\pm$0.01}	&\textbf{0.68$\pm$0.02}	&\textbf{1.24$\pm$0.06}	&\textbf{0.14$\pm$0.02}		&7579$\pm$524	   &2.12$\pm$0.14 \\\hline
\multirow{2}{*}{Insurance10}
&ELCS 	        &0.80$\pm$0.02	&0.67$\pm$0.02		&1.75$\pm$0.11	&0.23$\pm$0.01		&\textbf{10809$\pm$1528}	&\textbf{3.92$\pm$0.55}\\
&ELCS w/o N 	&\textbf{0.86$\pm$0.02}	&\textbf{0.73$\pm$0.02}	&\textbf{1.42$\pm$0.10}	&\textbf{0.17$\pm$0.01}		&11300$\pm$1713	   &4.49$\pm$0.55 \\\hline
\multirow{2}{*}{Child10}
&ELCS 	        &\textbf{0.83$\pm$0.05}	&\textbf{0.76$\pm$0.07}		&\textbf{0.73$\pm$0.20}	&0.14$\pm$0.03		&\textbf{13129$\pm$2613}	&\textbf{3.98$\pm$0.79}\\
&ELCS w/o N	        &0.83$\pm$0.06	&\textbf{0.76$\pm$0.07}	&0.74$\pm$0.21	&\textbf{0.13$\pm$0.03}		&13491$\pm$2559	   &4.32$\pm$0.77 \\\hline
\multirow{2}{*}{Pigs}
&ELCS 	        &\textbf{0.91$\pm$0.00}	&\textbf{0.99$\pm$0.00}		&0.42$\pm$0.02	&\textbf{0.15$\pm$0.01}		&\textbf{13374$\pm$8660}	&\textbf{8.91$\pm$6.84} \\
&ELCS w/o N	        &\textbf{0.91$\pm$0.00}	&\textbf{0.99$\pm$0.00}	&\textbf{0.40$\pm$0.02}	&\textbf{0.15$\pm$0.01}		&14343$\pm$8657	   &9.70$\pm$6.75 \\\hline
\multirow{2}{*}{Gene}
&ELCS 	        &\textbf{0.76$\pm$0.01}	&\textbf{0.79$\pm$0.01}		&\textbf{0.79$\pm$0.03}	&\textbf{0.32$\pm$0.01}		&\textbf{36950$\pm$7876}	&\textbf{11.03$\pm$2.35} \\
&ELCS w/o N 	        &\textbf{0.76$\pm$0.01}	&0.78$\pm$0.01	&\textbf{0.79$\pm$0.03}	&\textbf{0.32$\pm$0.01}		&38061$\pm$8041	   &12.52$\pm$2.62 \\\hline
\end{tabular}
\end{table}

\subsection{Why ELCS is Efficient and Effective?}
\label{Why}
In this section, we analyse the reason why ELCS is efficient and effective from the following two aspects. First, we give a case study to evaluate the effectiveness of utilizing  the N-structures to infer edge directions between a given variable and its children. Second, we evaluate the effectiveness of the proposed EMB subroutine, since the proposed ELCS algorithm relies on EMB.

\subsubsection{Case Study}

To illustrate the benefit of utilizing the N-structures, we do not use  the N-structures to distinguish  the children of a given variable in learning the MB of the variable, that is, in the DistinguishPC subroutine, we remove lines 8-12 in Algorithm \ref{algorithmIdentifyPC}, and we denote this version of ELCS as ``ECLS w/o N". Table \ref{caseNLocal5000} summarizes the experimental results of ECLS and ``ECLS w/o N" on the eight BNs with 5,000 data examples. From the experimental results, we note that ELCS achieves comparable performance against ``ECLS w/o N" in terms of ArrP, ArrR, SHD and FDR on average but ELCS performs less CI tests. Specifically, on Alarm, for each variable, the number of CI tests is reduced by 53 on average. On Gene, for each variable,  the number of CI tests is reduced by 1111 on average. We observe that there are significant differences between the number of CI tests of ELCS and ``ECLS w/o N" on Insurance, Alarm10, Insurance10, Pigs and Gene, but there is a little difference between them on the other BNs. The reason may be that ELCS makes use of  the N-structures to speed up children identification, and there are few N-structures existing in Alarm, Child and Child10, leading to that the number of CI tests of ELCS and ``ECLS w/o N" on these three BNs are  close to each other. The efficiency performance of ECLS will improve a little on the BNs with few N-structures, this is a limitation of ELCS. In summary, ELCS is more efficient and provides better local structure learning quality than ``ECLS w/o N", which indicates the effectiveness of leveraging the N-structures for  local casual structure learning.

\tabcolsep 0.03in
\begin{table}
\tiny
\caption{Comparison of EMB with State-of-the-art MB Learning Algorithms on Eight Benchmark BNs (size=5,000)}\label{EMB5000}
\centering
\begin{tabular}{c||c||cccccc}
\hline
Network & Algorithm  & Distance($\downarrow$) & F1($\uparrow$) & Precision($\uparrow$)   & Recall($\uparrow$)  &CI Tests($\downarrow$) &Time($\downarrow$)\\
\hline\hline
\multirow{7}{*}{Alarm}
&IAMB	        &0.15$\pm$0.03	&0.90$\pm$0.02		&0.94$\pm$0.02	&0.89$\pm$0.01		&\textbf{142$\pm$2}	&\textbf{0.05$\pm$0.00} \\
&MMMB 	        &0.10$\pm$0.02	&0.94$\pm$0.02		&0.92$\pm$0.02	&\textbf{0.97$\pm$0.01}		&604$\pm$26	&0.24$\pm$0.01 \\
&HITON-MB 	    &\textbf{0.06$\pm$0.02}	&\textbf{0.96$\pm$0.01}		&0.97$\pm$0.02	&\textbf{0.97$\pm$0.01}		&394$\pm$12	&0.13$\pm$0.00 \\
&STMB 	        &0.30$\pm$0.02	&0.78$\pm$0.02		&0.73$\pm$0.02	&0.96$\pm$0.01		&531$\pm$15	&0.19$\pm$0.01 \\
&BAMB 	        &0.09$\pm$0.03	&0.94$\pm$0.02		&0.96$\pm$0.03	&0.95$\pm$0.01		&351$\pm$11	&0.14$\pm$0.00\\
&EMB 	        &\textbf{0.06$\pm$0.01}	&\textbf{0.96$\pm$0.01}		&\textbf{0.99$\pm$0.01}	&0.95$\pm$0.01		&318$\pm$7	&0.10$\pm$0.00 \\
&EMB-II         &\textbf{0.06$\pm$0.01}	&\textbf{0.96$\pm$0.01}		&\textbf{0.99$\pm$0.01}	&0.95$\pm$0.01		&298$\pm$5	&0.09$\pm$0.00 \\\hline
\multirow{7}{*}{Insurance}
&IAMB	        &0.36$\pm$0.02	&0.76$\pm$0.01		&\textbf{0.94$\pm$0.02}	&0.67$\pm$0.01		&\textbf{104$\pm$2}	&\textbf{0.04$\pm$0.00} \\
&MMMB 	        &0.31$\pm$0.02	&0.79$\pm$0.02		&0.88$\pm$0.03	&0.76$\pm$0.02		&1186$\pm$124	&0.60$\pm$0.07 \\
&HITON-MB 	    &0.33$\pm$0.03	&0.78$\pm$0.02	&0.88$\pm$0.03	&0.74$\pm$0.02		&679$\pm$62	&0.31$\pm$0.04 \\
&STMB 	        &0.49$\pm$0.03	&0.65$\pm$0.02		&0.64$\pm$0.04	&\textbf{0.77$\pm$0.03}		&703$\pm$47	&0.33$\pm$0.02 \\
&BAMB 	        &\textbf{0.30$\pm$0.02}	&\textbf{0.80$\pm$0.01}		&0.89$\pm$0.03	&\textbf{0.77$\pm$0.02}		&619$\pm$39	&0.33$\pm$0.02 \\
&EMB 	        &0.31$\pm$0.01	&0.79$\pm$0.01	&0.92$\pm$0.02	&0.73$\pm$0.01		&370$\pm$23	&0.16$\pm$0.01 \\
&EMB-II	        &0.31$\pm$0.01	&0.79$\pm$0.01		&0.92$\pm$0.02	&0.73$\pm$0.01		&360$\pm$20	&0.15$\pm$0.01 \\\hline
\multirow{7}{*}{Child}
&IAMB	        &0.15$\pm$0.02	&0.90$\pm$0.02		&0.95$\pm$0.03	&0.88$\pm$0.01		&\textbf{63$\pm$1}	&\textbf{0.03$\pm$0.00} \\
&MMMB 	        &0.05$\pm$0.02	&0.97$\pm$0.01		&0.96$\pm$0.02	&\textbf{0.99$\pm$0.01}		&897$\pm$25	&0.47$\pm$0.01 \\
&HITON-MB 	    &\textbf{0.04$\pm$0.03}	&\textbf{0.98$\pm$0.02}		&\textbf{0.97$\pm$0.03} &\textbf{0.99$\pm$0.01}		&499$\pm$16	&0.24$\pm$0.01 \\
&STMB 	        &0.17$\pm$0.04	&0.89$\pm$0.03		&0.84$\pm$0.04	&0.98$\pm$0.02		&374$\pm$35	&0.17$\pm$0.02 \\
&BAMB 	        &0.09$\pm$0.03	&0.95$\pm$0.02		&0.93$\pm$0.02	&0.98$\pm$0.02		&376$\pm$11	&0.19$\pm$0.01  \\
&EMB 	        &0.05$\pm$0.02	&0.97$\pm$0.02	    &\textbf{0.97$\pm$0.02}	&0.98$\pm$0.02		&205$\pm$8	&0.09$\pm$0.00 \\
&EMB-II	        &0.05$\pm$0.02	&0.97$\pm$0.02	    &\textbf{0.97$\pm$0.02}	&0.98$\pm$0.02		&204$\pm$8	&0.09$\pm$0.00 \\\hline
\multirow{7}{*}{Alarm10}
&IAMB	        &0.36$\pm$0.01	&0.75$\pm$0.01		&0.83$\pm$0.01	&0.74$\pm$0.00		&\textbf{1637$\pm$14}	&0.59$\pm$0.01 \\
&MMMB 	        &0.26$\pm$0.01	&0.82$\pm$0.01		&0.88$\pm$0.01	&0.81$\pm$0.00		&1926$\pm$45    &0.60$\pm$0.01 \\
&HITON-MB 	    &\textbf{0.25$\pm$0.01}	&\textbf{0.84$\pm$0.01}		&0.90$\pm$0.01	&0.82$\pm$0.00		&1714$\pm$11	&\textbf{0.44$\pm$0.00} \\
&STMB 	        &0.67$\pm$0.01	&0.48$\pm$0.01	&0.41$\pm$0.01	&\textbf{0.83$\pm$0.01}		&5049$\pm$39	&1.89$\pm$0.02 \\
&BAMB 	        &0.30$\pm$0.01	&0.80$\pm$0.00		&0.83$\pm$0.01	&0.82$\pm$0.00		&1802$\pm$12	&0.57$\pm$0.01\\
&EMB 	        &\textbf{0.25$\pm$0.01}	&0.83$\pm$0.01		&\textbf{0.91$\pm$0.01}	&0.81$\pm$0.00		&1924$\pm$7	    &0.50$\pm$0.02 \\
&EMB-II	        &\textbf{0.25$\pm$0.01}	&0.83$\pm$0.01		&\textbf{0.91$\pm$0.01}	&0.81$\pm$0.00		&1908$\pm$7	    &0.49$\pm$0.02 \\\hline
\multirow{7}{*}{Insurance10}
&IAMB	        &0.42$\pm$0.01	&0.71$\pm$0.01		&0.89$\pm$0.01	&0.66$\pm$0.00		&\textbf{1210$\pm$8	}    &\textbf{0.50$\pm$0.01} \\
&MMMB 	        &0.33$\pm$0.01	&0.78$\pm$0.01		&0.82$\pm$0.01	&\textbf{0.80$\pm$0.00}		&3274$\pm$45    &1.53$\pm$0.03 \\
&HITON-MB 	    &0.32$\pm$0.01	&0.78$\pm$0.01		&0.84$\pm$0.01	&\textbf{0.80$\pm$0.00}		&2348$\pm$18	&0.93$\pm$0.01 \\
&STMB 	        &0.77$\pm$0.01	&0.40$\pm$0.01		&0.30$\pm$0.01	&0.79$\pm$0.00		&6781$\pm$118	&3.36$\pm$0.07 \\
&BAMB 	        &0.34$\pm$0.01	&0.77$\pm$0.00		&0.80$\pm$0.01	&\textbf{0.80$\pm$0.00} &2541$\pm$22			&1.17$\pm$0.01 \\
&EMB 	        &\textbf{0.28$\pm$0.01}	&\textbf{0.81$\pm$0.00}      &\textbf{0.91$\pm$0.01}	&0.78$\pm$0.00	&2189$\pm$15		&0.78$\pm$0.01	 \\
&EMB-II	        &\textbf{0.28$\pm$0.01}	&\textbf{0.81$\pm$0.00}      &\textbf{0.91$\pm$0.01}	&0.78$\pm$0.00	&2122$\pm$14		&0.75$\pm$0.01	 \\\hline
\multirow{7}{*}{Child10}
&IAMB	        &0.24$\pm$0.01	&0.84$\pm$0.01		&0.87$\pm$0.01	&0.88$\pm$0.00		&\textbf{750$\pm$10}	&\textbf{0.31$\pm$0.00} \\
&MMMB 	        &0.10$\pm$0.01	&0.94$\pm$0.01		&0.91$\pm$0.01	&\textbf{0.99$\pm$0.00}		&1622$\pm$21&0.70$\pm$0.01 \\
&HITON-MB 	    &0.08$\pm$0.01	&\textbf{0.95$\pm$0.01}		&0.92$\pm$0.01	&\textbf{0.99$\pm$0.00}		&1194$\pm$11	&0.43$\pm$0.01 \\
&STMB 	        &0.56$\pm$0.02	&0.56$\pm$0.02		&0.45$\pm$0.02	&\textbf{0.99$\pm$0.00}		&2881$\pm$24	&1.26$\pm$0.01 \\
&BAMB 	        &0.24$\pm$0.01	&0.84$\pm$0.00		&0.76$\pm$0.01	&\textbf{0.99$\pm$0.00}		&1111$\pm$10	&0.46$\pm$0.01 \\
&EMB 	        &\textbf{0.07$\pm$0.01}	&\textbf{0.95$\pm$0.01}		&\textbf{0.94$\pm$0.01}	&\textbf{0.99$\pm$0.00}		&1062$\pm$8	    &0.33$\pm$0.00 \\
&EMB-II	        &\textbf{0.07$\pm$0.01}	&\textbf{0.95$\pm$0.01}		&\textbf{0.94$\pm$0.01}	&\textbf{0.99$\pm$0.00}		&1061$\pm$8	    &0.33$\pm$0.00 \\\hline
\multirow{7}{*}{Pigs}
&IAMB	        &0.42$\pm$0.00	&0.71$\pm$0.00		&0.62$\pm$0.00	&0.96$\pm$0.00		&2616$\pm$7	&\textbf{1.28$\pm$0.00} \\
&MMMB 	        &0.13$\pm$0.00	&0.92$\pm$0.00      &0.87$\pm$0.00	&1.00$\pm$0.00		&3.2e5$\pm$2.2e4  &2.2e2$\pm$1.7e1\\
&HITON-MB 	    &0.14$\pm$0.00	&0.92$\pm$0.00		&0.86$\pm$0.00	&1.00$\pm$0.00		&46956$\pm$454	&34.94$\pm$0.33 \\
&STMB 	        &0.82$\pm$0.01	&0.26$\pm$0.00	&0.18$\pm$0.01	&1.00$\pm$0.00		&45770$\pm$2746	&29.20$\pm$2.05 \\
&BAMB 	        &0.18$\pm$0.01	&0.89$\pm$0.01		&0.82$\pm$0.01  &\textbf{1.00$\pm$0.00}			&29097$\pm$201	&31.16$\pm$0.13  \\
&EMB 	        &\textbf{0.12$\pm$0.00}	&\textbf{0.93$\pm$0.00}		&\textbf{0.88$\pm$0.00}	&\textbf{1.00$\pm$0.00}		&8784$\pm$3405	&5.73$\pm$2.43 \\
&EMB-II	        &\textbf{0.12$\pm$0.00}	&\textbf{0.93$\pm$0.00}		&\textbf{0.88$\pm$0.00}	&\textbf{1.00$\pm$0.00}		&7335$\pm$209	&4.94$\pm$0.26 \\\hline
\multirow{7}{*}{Gene}
&IAMB	        &0.32$\pm$0.00	&0.79$\pm$0.00		&0.76$\pm$0.00	&0.89$\pm$0.00		&\textbf{3463$\pm$10}	&1.36$\pm$0.00 \\
&MMMB 	        &\textbf{0.25$\pm$0.00}	&\textbf{0.83$\pm$0.00}		&\textbf{0.77$\pm$0.00}	&0.94$\pm$0.00		&6035$\pm$48     &2.21$\pm$0.02\\
&HITON-MB 	    &\textbf{0.25$\pm$0.00}	&\textbf{0.83$\pm$0.00}		&\textbf{0.77$\pm$0.00}	&0.94$\pm$0.00		&4576$\pm$22	&1.44$\pm$0.00 \\
&STMB 	        &0.88$\pm$0.00	&0.18$\pm$0.00		&0.13$\pm$0.00	&\textbf{1.00$\pm$0.00}		&17282$\pm$268	&7.70$\pm$0.18\\
&BAMB 	        &0.39$\pm$0.00	&0.73$\pm$0.00		&0.64$\pm$0.00	&0.94$\pm$0.00		&4474$\pm$30	&1.72$\pm$0.02 \\
&EMB 	        &0.26$\pm$0.00	&0.82$\pm$0.00		&0.76$\pm$0.00	&0.94$\pm$0.00		&4486$\pm$9	    &1.28$\pm$0.00 \\
&EMB-II	        &0.26$\pm$0.00	&0.82$\pm$0.00	&0.76$\pm$0.00	&0.94$\pm$0.00		&4412$\pm$11	    &\textbf{1.27$\pm$0.00} \\
\hline
\end{tabular}
\end{table}

\tabcolsep 0.03in
\begin{table}
\tiny
\caption{Comparison of EMB with State-of-the-art MB Learning Algorithms on Eight Benchmark BNs (size=1,000)}\label{EMB1000}
\centering
\begin{tabular}{c||c||cccccc}
\hline
Network & Algorithm  & Distance($\downarrow$) & F1($\uparrow$) & Precision($\uparrow$)   & Recall($\uparrow$)  &CI Tests($\downarrow$) &Time($\downarrow$)\\
\hline\hline
\multirow{7}{*}{Alarm}
&IAMB	        &0.27$\pm$0.01 &0.81$\pm$0.01	&0.93$\pm$0.01	&0.76$\pm$0.01	&\textbf{120$\pm$2}	&\textbf{0.02$\pm$0.00} \\
&MMMB 	        &0.20$\pm$0.01 &0.87$\pm$0.01	&0.91$\pm$0.02	&\textbf{0.87$\pm$0.01}	&437$\pm$33	&0.09$\pm$0.00\\
&HITON-MB 	    &\textbf{0.16$\pm$0.01} &\textbf{0.90$\pm$0.01}	&0.95$\pm$0.02	&\textbf{0.87$\pm$0.01}	&315$\pm$12	&0.05$\pm$0.00\\
&STMB 	        &0.39$\pm$0.03 &0.72$\pm$0.02	&0.71$\pm$0.02	&0.85$\pm$0.02	&392$\pm$12	&0.06$\pm$0.00\\
&BAMB 	        &0.21$\pm$0.01 &0.86$\pm$0.01	&0.91$\pm$0.02	&0.86$\pm$0.01	&280$\pm$16	&0.04$\pm$0.00		 \\
&EMB 	        &0.18$\pm$0.02 &0.88$\pm$0.01	&\textbf{0.96$\pm$0.02}	&0.85$\pm$0.02	&297$\pm$14	&0.04$\pm$0.00\\
&EMB-II	    &0.18$\pm$0.02 &0.88$\pm$0.01	&\textbf{0.96$\pm$0.02}	&0.85$\pm$0.02	&285$\pm$13	&0.04$\pm$0.00\\\hline
\multirow{7}{*}{Insurance}
&IAMB	        &0.48$\pm$0.02 &0.66$\pm$0.01	&\textbf{0.92$\pm$0.03}	&0.56$\pm$0.01	&\textbf{86$\pm$2}	&\textbf{0.01$\pm$0.00}	\\
&MMMB 	        &\textbf{0.42$\pm$0.03} &\textbf{0.71$\pm$0.02}	&0.83$\pm$0.03	&0.66$\pm$0.02	&511$\pm$47	&0.12$\pm$0.01\\
&HITON-MB 	    &0.45$\pm$0.03 &0.69$\pm$0.02	&0.83$\pm$0.04	&0.65$\pm$0.01	&358$\pm$34	&0.06$\pm$0.01\\
&STMB 	        &0.59$\pm$0.03 &0.58$\pm$0.03	&0.58$\pm$0.06	&0.66$\pm$0.04	&1138$\pm$1277	&0.15$\pm$0.16\\
&BAMB 	        &0.45$\pm$0.02 &0.69$\pm$0.02	&0.76$\pm$0.04	&\textbf{0.68$\pm$0.01}	&404$\pm$51	&0.06$\pm$0.01		 \\
&EMB 	        &0.46$\pm$0.03 &0.68$\pm$0.03	&0.81$\pm$0.08	&0.64$\pm$0.02	&395$\pm$82	&0.06$\pm$0.01 \\
&EMB-II	    &0.46$\pm$0.03 &0.68$\pm$0.03	&0.81$\pm$0.08	&0.64$\pm$0.02	&371$\pm$76	&0.05$\pm$0.01 \\\hline
\multirow{7}{*}{Child}
&IAMB	        &0.27$\pm$0.03	&0.82$\pm$0.02		&0.94$\pm$0.03	&0.76$\pm$0.02		&\textbf{54$\pm$1}	&\textbf{0.01$\pm$0.00} \\
&MMMB 	        &0.22$\pm$0.03	&0.85$\pm$0.02		&0.89$\pm$0.04	&0.86$\pm$0.02		&823$\pm$85  &0.13$\pm$0.01\\
&HITON-MB 	    &0.20$\pm$0.03  &0.87$\pm$0.02	&0.90$\pm$0.03	&0.87$\pm$0.02	&469$\pm$53	&0.06$\pm$0.01\\
&STMB 	        &0.23$\pm$0.07  &0.85$\pm$0.05	&0.86$\pm$0.05	&0.87$\pm$0.04	&221$\pm$7	&0.04$\pm$0.00\\
&BAMB 	        &0.23$\pm$0.04  &0.85$\pm$0.03	&0.84$\pm$0.04	&\textbf{0.91$\pm$0.01}	&441$\pm$58	&0.05$\pm$0.01		\\
&EMB 	        &\textbf{0.17+0.03}	&\textbf{0.89+0.03}		&\textbf{0.94$\pm$0.03}	&0.87$\pm$0.02		&320$\pm$52	&0.04$\pm$0.01 \\
&EMB-II	    &\textbf{0.17+0.03}	&\textbf{0.89+0.03}		&\textbf{0.94$\pm$0.03}	&0.87$\pm$0.02		&287$\pm$39	&0.04$\pm$0.00 \\\hline
\multirow{7}{*}{Alarm10}
&IAMB	        &0.52$\pm$0.01	&0.63$\pm$0.01		&0.78$\pm$0.01	&0.60$\pm$0.01		&1355$\pm$11	&0.18$\pm$0.00 \\
&MMMB 	        &0.39$\pm$0.01	&0.73$\pm$0.01		&0.84$\pm$0.01	&\textbf{0.70$\pm$0.00}		&1579$\pm$17    &0.28$\pm$0.00\\
&HITON-MB 	    &\textbf{0.37$\pm$0.01}	&\textbf{0.75$\pm$0.01}		&0.86$\pm$0.01	&\textbf{0.70$\pm$0.01}		&1474$\pm$13	&0.20$\pm$0.00 \\
&STMB 	        &0.76$\pm$0.01	&0.42$\pm$0.01		&0.37$\pm$0.02	&\textbf{0.70$\pm$0.01}		&3668$\pm$29	&0.76$\pm$0.00 \\
&BAMB 	        &0.46$\pm$0.00	&0.68$\pm$0.00		&0.74$\pm$0.01	&\textbf{0.70$\pm$0.00}		&1551$\pm$16	&0.23$\pm$0.00 \\
&EMB 	        &0.38$\pm$0.01	&0.74$\pm$0.00		&\textbf{0.88$\pm$0.01}	&0.69$\pm$0.00		&1765$\pm$11	&0.23$\pm$0.00 \\
&EMB-II	    &0.38$\pm$0.01	&0.74$\pm$0.00		&\textbf{0.88$\pm$0.01}	&0.69$\pm$0.00		&1756$\pm$12	&0.23$\pm$0.00 \\\hline			
\multirow{7}{*}{Insurance10}
&IAMB	        &0.55$\pm$0.01 &0.61$\pm$0.01	&0.85$\pm$0.01	&0.53$\pm$0.00	&\textbf{963$\pm$5}	    &\textbf{0.13$\pm$0.00} \\
&MMMB 	        &0.49$\pm$0.01 &0.66$\pm$0.01	&0.70$\pm$0.01	&0.70$\pm$0.01	&2180$\pm$48	&0.40$\pm$0.01\\
&HITON-MB 	    &\textbf{0.45$\pm$0.01} &\textbf{0.68$\pm$0.01}	&\textbf{0.74$\pm$0.01}	&0.70$\pm$0.01	&1698$\pm$28	&0.25$\pm$0.00\\
&STMB 	        &0.75$\pm$0.01	&0.46$\pm$0.01 &0.33$\pm$0.01			&0.74$\pm$0.01		&1443$\pm$14	&0.21$\pm$0.01 \\
&BAMB 	        &0.53$\pm$0.01	&0.62$\pm$0.01		&0.60$\pm$0.01	&\textbf{0.74$\pm$0.01}		&2243$\pm$59	&0.34$\pm$0.01 \\
&EMB 	        &0.65$\pm$0.06	&0.53$\pm$0.05	&0.47$\pm$0.07	&0.71$\pm$0.01		&4333$\pm$1736	&0.57$\pm$0.23 \\
&EMB-II 	        &0.65$\pm$0.06	&0.53$\pm$0.05		&0.47$\pm$0.07	&0.71$\pm$0.01		&3966$\pm$1423	&0.55$\pm$0.23 \\\hline
\multirow{7}{*}{Child10}
&IAMB	        &0.40$\pm$0.01 &0.72$\pm$0.01	&0.84$\pm$0.01	&0.71$\pm$0.01	&\textbf{614$\pm$8}	    &\textbf{0.08$\pm$0.00} \\
&MMMB 	        &0.28$\pm$0.02 &0.81$\pm$0.01	&0.82$\pm$0.02	&0.86$\pm$0.01	&1757$\pm$43	&0.27$\pm$0.00\\
&HITON-MB 	    &0.25$\pm$0.02 &0.83$\pm$0.01   &0.84$\pm$0.02	&0.87$\pm$0.01	&1272$\pm$24	&0.17$\pm$0.00 \\
&STMB 	        &0.66$\pm$0.02 &0.48$\pm$0.02	&0.39$\pm$0.02	&0.85$\pm$0.01	&2186$\pm$41	&0.39$\pm$0.00\\
&BAMB 	        &0.47$\pm$0.01 &0.67$\pm$0.01	&0.58$\pm$0.01	&\textbf{0.90$\pm$0.01}	&1460$\pm$43	    &0.20$\pm$0.01	\\
&EMB 	        &\textbf{0.22$\pm$0.02} &\textbf{0.85$\pm$0.01}	&\textbf{0.87$\pm$0.02}	&0.88$\pm$0.01	&1225$\pm$10	&0.16$\pm$0.00 \\
&EMB-II	    &\textbf{0.22$\pm$0.02} &\textbf{0.85$\pm$0.01}	&\textbf{0.87$\pm$0.02}	&0.88$\pm$0.01	&1182$\pm$8	&0.16$\pm$0.00 \\\hline
\multirow{7}{*}{Pigs}
&IAMB	        &0.34$\pm$0.00	&0.79$\pm$0.00		&0.82$\pm$0.00	&0.84$\pm$0.00		&\textbf{1755$\pm$1}	&\textbf{0.22$\pm$0.00} \\
&MMMB 	        &0.15$\pm$0.01	&0.91$\pm$0.01		&0.85$\pm$0.01	&1.00$\pm$0.00		&197884$\pm$17879  &8.44$\pm$0.71\\
&HITON-MB 	    &0.12$\pm$0.01	&0.92$\pm$0.01		&0.88$\pm$0.01	&\textbf{1.00$\pm$0.00}		&47028$\pm$1667	&4.78$\pm$0.18 \\
&STMB 	        &0.85$\pm$0.00	&0.25$\pm$0.00		&0.15$\pm$0.00	&\textbf{1.00$\pm$0.00}		&25626$\pm$3234	&2.24$\pm$0.09 \\
&BAMB 	        &0.31$\pm$0.01	&0.80$\pm$0.01		&0.69$\pm$0.01	&\textbf{1.00$\pm$0.00}		&40466$\pm$4419	&11.22$\pm$1.55 \\
&EMB 	        &\textbf{0.11$\pm$0.01}	&\textbf{0.93$\pm$0.00}		&\textbf{0.89$\pm$0.01}	&\textbf{1.00$\pm$0.00}		&7541$\pm$99	&0.58$\pm$0.01 \\
&EMB-II	        &\textbf{0.11$\pm$0.01}	&\textbf{0.93$\pm$0.00}		&\textbf{0.89$\pm$0.01}	&\textbf{1.00$\pm$0.00}		&7466$\pm$193	&0.56$\pm$0.01 \\\hline
\multirow{7}{*}{Gene}
&IAMB	        &0.39$\pm$0.00	&0.73$\pm$0.00	&\textbf{0.79$\pm$0.00}	&0.79$\pm$0.00		&\textbf{2887$\pm$10}	&\textbf{0.36$\pm$0.00} \\
&MMMB 	        &0.28$\pm$0.00	&0.81$\pm$0.00		&0.75$\pm$0.00	&0.93$\pm$0.00		&4569$\pm$36	&0.75$\pm$0.01\\
&HITON-MB 	    &\textbf{0.25$\pm$0.01}	&\textbf{0.83$\pm$0.00}		&\textbf{0.79$\pm$0.00}	&0.93$\pm$0.00		&3918$\pm$26	&0.54$\pm$0.01 \\
&STMB 	        &0.86$\pm$0.00	&0.21$\pm$0.00		&0.14$\pm$0.00	&\textbf{0.99$\pm$0.00}		&10672$\pm$74	&2.50$\pm$0.01\\
&BAMB 	        &0.46$\pm$0.01	&0.67$\pm$0.00		&0.57$\pm$0.01	&0.94$\pm$0.00		&3817$\pm$52	&0.61$\pm$0.01 \\
&EMB 	        &\textbf{0.25$\pm$0.00}	&\textbf{0.83$\pm$0.00}		&0.78$\pm$0.00	&0.93$\pm$0.00		&4228$\pm$18	&0.57$\pm$0.00 \\
&EMB-II	        &\textbf{0.25$\pm$0.00}	&\textbf{0.83$\pm$0.00}	&0.78$\pm$0.00	&0.93$\pm$0.00		&4204$\pm$17	&0.57$\pm$0.00 \\
\hline
\end{tabular}
\end{table}
 \subsubsection{Experimental Results of EMB and Its Rivals}
 We evaluate the effectiveness of  the proposed EMB by comparing it with five state-of-the-art MB learning algorithms, including  BAMB \cite{BAMB}, IAMB \cite{IAMB}, MMMB \cite{TsamardinosAS03}, HITON-MB \cite{AliferisTS03} and  STMB \cite{GaoJ17}.
\par For MB learning algorithms, we use precision, recall, F1,   distance \cite{PCMB} \cite{GaoJ17}, CI tests, and running time (in seconds) as the evaluation metrics.
\begin{itemize}
\item \emph{Precision}: The number of true positives in the output (i.e., the variables in the output belonging to the true MB  of a target variable in a test DAG) divided by the number of variables in the output of an algorithm.
\item \emph{Recall}:   The number of true positives in the output divided by the number of true positives (the
number of the true MB  of a target variable in a test DAG).
\item \emph{F1 = 2 * Precision * Recall / (Precision + Recall)}: The \emph{F1} score is the harmonic average of the precision and recall,  where \emph{F1} = 1 is the best case (perfect precision and recall) while \emph{F1} = 0 is the worst case.
    \item \emph{Distance} = $\sqrt{(1-Precision)^2 +(1-Recall)^2}$ \cite{PCMB} \cite{GaoJ17},  where \emph{distance} = 0 is the best case (perfect precision and recall) while \emph{distance} = $\sqrt{2}$ is the worst case.
\item \emph{Efficiency}: The number of CI tests and the running time (in seconds) are used to measure the efficiency.
\end{itemize}

\par Tables \ref{EMB5000}-\ref{EMB1000} report the experimental results of EMB and its rivals. From the experimental results, we have the following observations.

\par \textbf{EMB \emph{versus}  IAMB, MMMB and HITON-MB.} IAMB is  much faster  than EMB, while IAMB is significantly worse than  EMB in terms of distance, F1, precision and recall on average.  Compared with MMMB and HITON-MB, EMB is more efficient. EMB needs much less CI tests than MMMB and HITON-MB. In addition, using 5,000 data samples, EMB is 2 times faster than MMMB and 1.2 times faster than  HITON-MB on  average. Moreover, EMB is more accurate than MMMB. In particular, using 5,000 data samples, EMB achieves  the lowest distance and the highest F1 values on Alarm, Insurance10, Child10 and Pigs. Using 1,000 data samples, EMB  obtains the lowest distance and the highest F1 values on Child, Child10, Pigs and Gene. Overall, EMB is superior to  IAMB, MMMB and HITON-MB.


\par \textbf{EMB \emph{versus}   STMB and BAMB.}  From Tables \ref{EMB5000}-\ref{EMB1000}, we note that STMB achieves higher recall values than EMB, but on the distance, F1 and precision metrics, STMB is significantly worse than EMB. Compared with BAMB, EMB achieves lower distance and higher F1 values. Additionally, the number of CI tests of EMB is less than STMB and BAMB. More specifically, using 5,000 data samples, EMB is 3.5 times faster than STMB and 1.5 times faster than  BAMB on  average. In a word, EMB performs better than BAMB and STMB in both efficiency and accuracy.

\par \textbf{EMB \emph{versus} EMB-II.} EMB is inferior to EMB-II. Compared with EMB,  EMB-II uses less CI tests for MB learning while achieving the same performance as measured by the distance, F1, precision and recall matrics, which indicates the efficiency of EMB-II.

\par  EMB is able to effectively find the MB of a target variable, and simultaneously distinguish parents and children of the target variable.  We note  that EMB uses less CI tests for MB learning, which can reduce the impact of unreliable CI tests. In summary,  EMB is helpful to learn the local causal structure.

\par To further demonstrate the effectiveness of EMB, we propose three variants of ELCS, which are referred to as ELCS-M, ECLS-S, ELCS-B, respectively. ELCS-M uses MMMB to replace EMB in ELCS. ECLS-S and ELCS-B use STMB and BAMB to replace EMB in ELCS, respectively. Table \ref{caseLocal5000} reports the experimental results of ECLS and its three variants on the eight BNs with 5,000 data examples. From the table, we observe that ELCS  outperforms these three rivals in terms of both CI tests and running time, which implies the efficiency of ELCS. We also note that ELCS achieves better ArrP, ArrR, SHD and FDR values than that of these three rivals, which shows the effectiveness of ELCS.

\par EMB has achieved encouraging performance, but it still suffers from the following two drawbacks. First, to improve the efficiency of EMB while maintaining competitive performance, EMB chooses to remove non-spouses within \emph{\textbf{U}}$\setminus$\{\emph{T}\}$\setminus$\textbf{\emph{PC}}$_{\emph{T}}$  of the target variable \emph{T} as early as possible at lines 2-16 in Algorithm \ref{FindSpouses}. The size of the conditioning sets \emph{\textbf{Temp}} (line 9 in Algorithm \ref{FindSpouses}) and \{\emph{Y}\} $\cup$ \textbf{\emph{Sep}}$_{\emph{T}}$\{\emph{X}\} (line 11 in Algorithm \ref{FindSpouses}) may be large, when the size of data samples is finite, the results of CI tests may be unreliable, leading to poor performance of EMB. Second, the performance of EMB is limited by HITON-PC  that is used for PC learning. If HITON-PC has a lower quality of PC learning, inaccurate MBs will be learnt by EMB.

\tabcolsep 0.03in
\begin{table}
\tiny
\caption{Comparison of ELCS with ELCS-M, ECLS-S and ELCS-B on Eight Benchmark BNs (size=5,000)}\label{caseLocal5000}
\centering
\begin{tabular}{c||c||cccccc}\hline
Network & Algorithm  & ArrP($\uparrow$)   & ArrR($\uparrow$) &SHD($\downarrow$)   &FDR($\downarrow$) &CI Tests($\downarrow$) &Time($\downarrow$) \\\hline\hline
\multirow{5}{*}{Alarm}
&ELCS-M 	        &0.71$\pm$0.02 &0.59$\pm$0.03 &1.14$\pm$0.06 &0.26$\pm$0.03 &1627$\pm$186 &0.62$\pm$0.08 \\
&ELCS-S 	        &0.79$\pm$0.03 &0.72$\pm$0.05 &0.88$\pm$0.14 &0.13$\pm$0.04 &1402$\pm$78 &0.57$\pm$0.03 \\
&ELCS-B 	        &0.78$\pm$0.03 &0.66$\pm$0.05 &0.92$\pm$0.09 &0.26$\pm$0.03 &778$\pm$46 &0.30$\pm$0.02 \\
&ELCS 	        &\textbf{0.86$\pm$0.01}	&\textbf{0.81$\pm$0.01}		&\textbf{0.44$\pm$0.06}	&\textbf{0.07$\pm$0.02}		&648$\pm$55	    &0.20$\pm$0.02 \\
&ELCS-II	    &\textbf{0.86$\pm$0.01}	&\textbf{0.81$\pm$0.01}		&\textbf{0.44$\pm$0.06}	&\textbf{0.07$\pm$0.02}		&\textbf{607$\pm$52}	    &\textbf{0.19$\pm$0.02} \\\hline
\multirow{5}{*}{Insurance}
&ELCS-M 	        &0.76$\pm$0.03 &0.59$\pm$0.03 &1.87$\pm$0.15 &0.31$\pm$0.04 &6976$\pm$1183 &3.28$\pm$0.62 \\
&ELCS-S 	        &0.67$\pm$0.05 &0.50$\pm$0.05 &2.50$\pm$0.25 &0.37$\pm$0.04 &2653$\pm$476 &1.39$\pm$0.25 \\
&ELCS-B 	        &0.67$\pm$0.02 &0.44$\pm$0.02 &2.26$\pm$0.10 &0.41$\pm$0.01 &3182$\pm$447 &1.78$\pm$0.27 \\
&ELCS 	        &\textbf{0.85$\pm$0.04}	&\textbf{0.69$\pm$0.04}		&\textbf{1.61$\pm$0.06}	&\textbf{0.18$\pm$0.05}		&1686$\pm$276	&\textbf{0.75$\pm$0.12} \\
&ELCS-II 	    &\textbf{0.85$\pm$0.04}	&\textbf{0.69$\pm$0.04}		&\textbf{1.61$\pm$0.06}	&\textbf{0.18$\pm$0.05}		&\textbf{1637$\pm$275}	&\textbf{0.75$\pm$0.12} \\\hline
\multirow{5}{*}{Child}
&ELCS-M 	        &0.68$\pm$0.11 &0.56$\pm$0.12 &1.18$\pm$0.31 &0.06$\pm$0.04 &8897$\pm$1247 &4.43$\pm$0.62 \\
&ELCS-S 	        &\textbf{0.81$\pm$0.07} &\textbf{0.72$\pm$0.09} &\textbf{0.75$\pm$0.18} &0.18$\pm$0.07 &2451$\pm$516 &1.28$\pm$0.28 \\
&ELCS-B	        &0.75$\pm$0.03 &0.65$\pm$0.05 &0.93$\pm$0.14 &0.26$\pm$0.05 &2252$\pm$195 &1.16$\pm$0.10 \\
&ELCS 	        &0.71$\pm$0.12	&0.61$\pm$0.16		&1.08$\pm$0.36	&0.09$\pm$0.08		&2093$\pm$287	&\textbf{0.93$\pm$0.10} \\
&ELCS-II 	    &0.71$\pm$0.12	&0.61$\pm$0.16		&1.08$\pm$0.36	&\textbf{0.09$\pm$0.08}		&\textbf{2087$\pm$287}	&\textbf{0.93$\pm$0.10} \\\hline
\multirow{5}{*}{Alarm10}
&ELCS-M 	        &0.77$\pm$0.01 &0.59$\pm$0.02 &1.60$\pm$0.07 &0.19$\pm$0.02 &10570$\pm$1207 &3.17$\pm$0.39 \\
&ELCS-S 	        &0.69$\pm$0.01 &0.46$\pm$0.01 &1.65$\pm$0.03 &0.32$\pm$0.01 &8429$\pm$237 &3.66$\pm$0.15 \\
&ELCS-B 	        &0.68$\pm$0.01 &0.44$\pm$0.00 &1.85$\pm$0.02 &0.46$\pm$0.01 &7223$\pm$170 &2.09$\pm$0.03 \\
&ELCS	        &\textbf{0.83$\pm$0.01}	&\textbf{0.68$\pm$0.02}		&\textbf{1.26$\pm$0.07}	&\textbf{0.14$\pm$0.02}		&\textbf{6893$\pm$483}	&\textbf{1.77$\pm$0.12}\\
&ELCS-II	    &\textbf{0.83$\pm$0.01}	&\textbf{0.68$\pm$0.02}		&\textbf{1.26$\pm$0.07}	&\textbf{0.14$\pm$0.02}		&6916$\pm$480	&\textbf{1.77$\pm$0.12}\\\hline
\multirow{5}{*}{Insurance10}
&ELCS-M	        &0.75$\pm$0.01 &0.61$\pm$0.01 &1.85$\pm$0.04 &0.33$\pm$0.01 &14461$\pm$2927 &6.35$\pm$1.27 \\
&ELCS-S 	        &0.60$\pm$0.01 &0.38$\pm$0.01 &2.44$\pm$0.02 &0.51$\pm$0.01 &10338$\pm$533 &5.78$\pm$0.55 \\
&ELCS-B 	        &0.64$\pm$0.01 &0.42$\pm$0.01 &2.31$\pm$0.06 &0.47$\pm$0.01 &\textbf{7556$\pm$224} &\textbf{3.55$\pm$0.11} \\
&ELCS 	        &\textbf{0.80$\pm$0.02}	&\textbf{0.67$\pm$0.02}		&\textbf{1.75$\pm$0.11}	&\textbf{0.23$\pm$0.01}		&10809$\pm$1528	&3.92$\pm$0.55\\
&ELCS-II 	    &\textbf{0.80$\pm$0.02}	&\textbf{0.67$\pm$0.02}		&\textbf{1.75$\pm$0.11}	&\textbf{0.23$\pm$0.01}		&10605$\pm$1499	&3.91$\pm$0.55\\\hline
\multirow{5}{*}{Child10}
&ELCS-M 	        &0.80$\pm$0.02 &0.75$\pm$0.03 &0.75$\pm$0.08 &0.16$\pm$0.02 &13438$\pm$1388 &5.80$\pm$0.58 \\
&ELCS-S 	        &0.66$\pm$0.01 &0.48$\pm$0.02 &1.17$\pm$0.03 &0.45$\pm$0.02 &15249$\pm$2335 &6.86$\pm$0.70 \\
&ELCS-B 	        &0.70$\pm$0.01 &0.52$\pm$0.02 &1.05$\pm$0.05 &0.47$\pm$0.02 &13880$\pm$1187 &4.28$\pm$0.05 \\
&ELCS 	        &\textbf{0.83$\pm$0.05}	&\textbf{0.76$\pm$0.07}		&\textbf{0.73$\pm$0.20}	&\textbf{0.14$\pm$0.03}		&13129$\pm$2613	&\textbf{3.98$\pm$0.79}\\
&ELCS-II	        &\textbf{0.83$\pm$0.05}	&\textbf{0.76$\pm$0.07}		&\textbf{0.73$\pm$0.20}	&\textbf{0.14$\pm$0.03}		&\textbf{13104$\pm$2605}	&\textbf{3.98$\pm$0.79}\\\hline
\multirow{5}{*}{Pigs}
&ELCS-M 	        &-	&-		&-	&-		&-	&-  \\
&ELCS-S 	        &-	&-		&-	&-		&-	&-  \\
&ELCS-B	        &-	&-		&-	&-		&-	&-  \\
&ELCS	        &\textbf{0.91$\pm$0.00}	&\textbf{0.99$\pm$0.00}		&\textbf{0.42$\pm$0.02}	&\textbf{0.15$\pm$0.01}		&13374$\pm$8660	&8.91$\pm$6.84 \\
&ELCS-II 	        &\textbf{0.91$\pm$0.00}	&\textbf{0.99$\pm$0.00}		&\textbf{0.42$\pm$0.02}	&\textbf{0.15$\pm$0.01}		&\textbf{11467$\pm$5659} &\textbf{8.38$\pm$5.12} \\\hline
\multirow{5}{*}{Gene}
&ELCS-M 	        &-	&-		&-	&-		&-	&-  \\
&ELCS-S 	        &-	&-		&-	&-		&-	&-  \\
&ELCS-B 	        &-	&-		&-	&-		&-	&-  \\
&ELCS 	            &\textbf{0.76$\pm$0.01}	&\textbf{0.79$\pm$0.01}		&\textbf{0.79$\pm$0.03}	&\textbf{0.32$\pm$0.01}		&36950$\pm$7876	&11.03$\pm$2.35 \\
&ELCS-II 	        &\textbf{0.76$\pm$0.01}	&\textbf{0.79$\pm$0.01}		&\textbf{0.79$\pm$0.03}	&\textbf{0.32$\pm$0.01}		&\textbf{36051$\pm$7696}	&\textbf{11.02$\pm$2.08} \\\hline
\end{tabular}
\end{table}

\section{Conclusion}
\label{conclusion}
 A new local causal structure learning algorithm (ELCS) has been proposed in this paper, which reduces the search space in distinguishing parents from children of a  target variable of interest. Specifically, ELCS makes  use of the N-structures  to distinguish  parents  from children of the target variable during learning the MB of the target variable. Furthermore, to combine MB learning with the N-structures to infer  edge directions between the target variable and its PC, we design an effective MB discovery  subroutine (EMB). We theoretically analyze the correctness of ELCS.  Extensive experimental results on benchmark BNs indicate that  ELCS not only improves the efficiency for learning the local causal structure, but also  achieves better performance in accuracy. In future, we plan to extend the ELCS algorithm for global causal structures learning and robust machine learning.


%

%

\ifCLASSOPTIONcompsoc
  \section*{Acknowledgments}
  This work was supported  in part by the National Key Research and Development Program of China (under Grant 2020AAA0106100), the National Natural Science Foundation of China (under Grant 61876206), and  Open Project Foundation of Intelligent Information Processing Key Laboratory of Shanxi Province (under Grant CICIP2020003).
\else
  \section*{Acknowledgment}
\fi


\ifCLASSOPTIONcaptionsoff
  \newpage
\fi

\bibliographystyle{IEEEtran}
\bibliography{Refs}

\begin{thebibliography}{10}
\providecommand{\url}[1]{#1}
\csname url@samestyle\endcsname
\providecommand{\newblock}{\relax}
\providecommand{\bibinfo}[2]{#2}
\providecommand{\BIBentrySTDinterwordspacing}{\spaceskip=0pt\relax}
\providecommand{\BIBentryALTinterwordstretchfactor}{4}
\providecommand{\BIBentryALTinterwordspacing}{\spaceskip=\fontdimen2\font plus
\BIBentryALTinterwordstretchfactor\fontdimen3\font minus
  \fontdimen4\font\relax}
\providecommand{\BIBforeignlanguage}[2]{{%
\expandafter\ifx\csname l@#1\endcsname\relax
\typeout{** WARNING: IEEEtran.bst: No hyphenation pattern has been}%
\typeout{** loaded for the language `#1'. Using the pattern for}%
\typeout{** the default language instead.}%
\else
\language=\csname l@#1\endcsname
\fi
#2}}
\providecommand{\BIBdecl}{\relax}
\BIBdecl

\bibitem{YuTKDD20}
K.~Yu, L.~Liu, and J.~Li, ``A unified view of causal and non-causal feature
  selection,'' \emph{ACM Transactions on Knowledge Discovery from Data (TKDD),
  in press}, 2020.

\bibitem{CaiQZZH19}
R.~Cai, J.~Qiao, K.~Zhang, Z.~Zhang, and Z.~Hao, ``Causal discovery with
  cascade nonlinear additive noise model,'' in \emph{Proceedings of the
  Twenty-Eighth International Joint Conference on Artificial Intelligence,
  {IJCAI}, Macao, China, 10-16 August}, 2019, pp. 1609--1615.

\bibitem{Huang0GG19}
B.~Huang, K.~Zhang, M.~Gong, and C.~Glymour, ``Causal discovery and forecasting
  in nonstationary environments with state-space models,'' in \emph{Proceedings
  of the 36th International Conference on Machine Learning, {ICML}, 9-15 June,
  Long Beach, California, {USA}}, 2019, pp. 2901--2910.

\bibitem{MarxV19a}
A.~Marx and J.~Vreeken, ``Identifiability of cause and effect using regularized
  regression,'' in \emph{Proceedings of the 25th {ACM} {SIGKDD} International
  Conference on Knowledge Discovery {\&} Data Mining, {KDD} 2019, Anchorage,
  AK, USA, 4-8 August}, 2019, pp. 852--861.

\bibitem{yu2019multi}
K.~Yu, L.~Liu, J.~Li, W.~Ding, and T.~D. Le, ``Multi-source causal feature
  selection,'' \emph{IEEE Transactions on Pattern Analysis and Machine
  Intelligence}, vol.~42, no.~9, pp. 2240--2256, 2020.

\bibitem{GourevitchBF06}
B.~Gour{\'{e}}vitch, R.~L. Bouquin{-}Jeann{\`{e}}s, and G.~Faucon, ``Linear and
  nonlinear causality between signals: methods, examples and neurophysiological
  applications,'' \emph{Biol. Cybern.}, vol.~95, no.~4, pp. 349--369, 2006.

\bibitem{ChoiCN20}
J.~Choi, R.~S. Chapkin, and Y.~Ni, ``Bayesian causal structural learning with
  zero-inflated poisson bayesian networks,'' in \emph{Advances in Neural
  Information Processing Systems 33: Annual Conference on Neural Information
  Processing Systems, December 6-12, virtual}, 2020.

\bibitem{maathuis2009}
M.~H. Maathuis, M.~Kalisch, P.~B{\"u}hlmann \emph{et~al.}, ``Estimating
  high-dimensional intervention effects from observational data,'' \emph{The
  Annals of Statistics}, vol.~37, no.~6A, pp. 3133--3164, 2009.

\bibitem{pearl2009causality}
J.~Pearl, \emph{Causality}.\hskip 1em plus 0.5em minus 0.4em\relax Cambridge
  university press, 2009.

\bibitem{MarxV19b}
A.~Marx and J.~Vreeken, ``Testing conditional independence on discrete data
  using stochastic complexity,'' in \emph{The 22nd International Conference on
  Artificial Intelligence and Statistics, {AISTATS}, 16-18 April, Naha,
  Okinawa, Japan}, 2019, pp. 496--505.

\bibitem{TsamardinosBA06}
I.~Tsamardinos, L.~E. Brown, and C.~F. Aliferis, ``The max-min hill-climbing
  bayesian network structure learning algorithm,'' \emph{Machine Learning},
  vol.~65, no.~1, pp. 31--78, 2006.

\bibitem{ZhengARX18}
X.~Zheng, B.~Aragam, P.~Ravikumar, and E.~P. Xing, ``Dags with {NO} {TEARS:}
  continuous optimization for structure learning,'' in \emph{Advances in Neural
  Information Processing Systems 31: Annual Conference on Neural Information
  Processing Systems 2018, NeurIPS 2018, 3-8 December, Montr{\'{e}}al,
  Canada.}, 2018, pp. 9492--9503.

\bibitem{YuCGY19}
Y.~Yu, J.~Chen, T.~Gao, and M.~Yu, ``{DAG-GNN:} {DAG} structure learning with
  graph neural networks,'' in \emph{Proceedings of the 36th International
  Conference on Machine Learning, {ICML}, 9-15 June, Long Beach, California,
  {USA}}, 2019, pp. 7154--7163.

\bibitem{YinZWHZG08}
J.~Yin, Y.~Zhou, C.~Wang, P.~He, C.~Zheng, and Z.~Geng, ``Partial orientation
  and local structural learning of causal networks for prediction,'' in
  \emph{Causation and Prediction Challenge at {WCCI}, Hong Kong, June 1-6},
  2008, pp. 93--105.

\bibitem{GaoJ15}
T.~Gao and Q.~Ji, ``Local causal discovery of direct causes and effects,'' in
  \emph{Advances in Neural Information Processing Systems 28: Annual Conference
  on Neural Information Processing Systems, 7-12 December, Montreal, Quebec,
  Canada}, 2015, pp. 2512--2520.

\bibitem{CCMB}
X.~Wu, B.~Jiang, K.~Yu, C.~Miao, and H.~Chen, ``Accurate markov boundary
  discovery for causal feature selection,'' \emph{{IEEE} Trans. Cybern.},
  vol.~50, no.~12, pp. 4983--4996, 2020.

\bibitem{BAMB}
Z.~Ling, K.~Yu, H.~Wang, L.~Liu, W.~Ding, and X.~Wu, ``{BAMB:} {A} balanced
  markov blanket discovery approach to feature selection,'' \emph{ACM
  Transactions on Intelligent Systems and Technology}, vol.~10, no.~5, pp.
  52:1--52:25, 2019.

\bibitem{NiinimakiP12}
T.~Niinimaki and P.~Parviainen, ``Local structure discovery in bayesian
  networks,'' in \emph{Proceedings of the Twenty-Eighth Conference on
  Uncertainty in Artificial Intelligence, Catalina Island, CA, USA, 14-18
  August}, 2012, pp. 634--643.

\bibitem{S2TMB}
T.~Gao and Q.~Ji, ``Efficient score-based markov blanket discovery,''
  \emph{Int. J. Approx. Reason.}, vol.~80, pp. 277--293, 2017.

\bibitem{GSMB}
D.~Margaritis and S.~Thrun, ``Bayesian network induction via local
  neighborhoods,'' in \emph{Advances in Neural Information Processing Systems
  12, {[NIPS} Conference, Denver, Colorado, USA, November 29 - December 4,
  1999]}, S.~A. Solla, T.~K. Leen, and K.~M{\"{u}}ller, Eds., 1999, pp.
  505--511.

\bibitem{IAMB}
I.~Tsamardinos and C.~F. Aliferis, ``Towards principled feature selection:
  Relevancy, filters and wrappers,'' in \emph{Proceedings of the Ninth
  International Workshop on Artificial Intelligence and Statistics, {AISTATS}
  2003, Key West, Florida, USA, January 3-6, 2003}, C.~M. Bishop and B.~J.
  Frey, Eds., 2003.

\bibitem{tsamardinos2003algorithms}
I.~Tsamardinos, C.~F. Aliferis, A.~R. Statnikov, and E.~Statnikov, ``Algorithms
  for large scale markov blanket discovery.'' in \emph{FLAIRS conference},
  vol.~2, 2003, pp. 376--380.

\bibitem{FastIAMB}
S.~Yaramakala and D.~Margaritis, ``Speculative markov blanket discovery for
  optimal feature selection,'' in \emph{Proceedings of the 5th {IEEE}
  International Conference on Data Mining {(ICDM} 2005), 27-30 November 2005,
  Houston, Texas, {USA}}, 2005, pp. 809--812.

\bibitem{TsamardinosAS03}
I.~Tsamardinos, C.~F. Aliferis, and A.~R. Statnikov, ``Time and sample
  efficient discovery of markov blankets and direct causal relations,'' in
  \emph{Proceedings of the Ninth {ACM} {SIGKDD} International Conference on
  Knowledge Discovery and Data Mining, Washington, DC, USA, 24-27 August},
  2003, pp. 673--678.

\bibitem{AliferisTS03}
C.~F. Aliferis, I.~Tsamardinos, and A.~R. Statnikov, ``{HITON:} {A} novel
  markov blanket algorithm for optimal variable selection,'' in \emph{{AMIA},
  American Medical Informatics Association Annual Symposium, Washington, DC,
  USA, November 8-12}, 2003.

\bibitem{PCMB}
J.~M. Pe{\~{n}}a, R.~Nilsson, J.~Bj{\"{o}}rkegren, and J.~Tegn{\'{e}}r,
  ``Towards scalable and data efficient learning of markov boundaries,''
  \emph{Int. J. Approx. Reason.}, vol.~45, no.~2, pp. 211--232, 2007.

\bibitem{GaoJ17}
T.~Gao and Q.~Ji, ``Efficient markov blanket discovery and its application,''
  \emph{{IEEE} Trans. Cybernetics}, vol.~47, no.~5, pp. 1169--1179, 2017.

\bibitem{yu2020causality}
K.~Yu, X.~Guo, L.~Liu, J.~Li, H.~Wang, Z.~Ling, and X.~Wu, ``Causality-based
  feature selection: Methods and evaluations,'' \emph{ACM Computing Surveys
  (CSUR)}, vol.~53, no.~5, pp. 1--36, 2020.

\bibitem{MargaritisT99}
D.~Margaritis and S.~Thrun, ``Bayesian network induction via local
  neighborhoods,'' in \emph{Advances in Neural Information Processing Systems
  12, {[NIPS} Conference, Denver, Colorado, USA, November 29-December 4]},
  1999, pp. 505--511.

\bibitem{GaoFC17}
T.~Gao, K.~P. Fadnis, and M.~Campbell, ``Local-to-global bayesian network
  structure learning,'' in \emph{Proceedings of the 34th International
  Conference on Machine Learning, {ICML}, Sydney, NSW, Australia, 6-11 August},
  2017, pp. 1193--1202.

\bibitem{GaoW18}
T.~Gao and D.~Wei, ``Parallel bayesian network structure learning,'' in
  \emph{Proceedings of the 35th International Conference on Machine Learning,
  {ICML} 2018, Stockholmsm{\"{a}}ssan, Stockholm, Sweden, 10-15 July, 2018},
  2018, pp. 1671--1680.

\bibitem{Zhu1906}
S.~Zhu, I.~Ng, and Z.~Chen, ``Causal discovery with reinforcement learning,''
  in \emph{8th International Conference on Learning Representations, {ICLR}
  2020, April 26-30 Addis Ababa, Ethiopia}, 2020.

\bibitem{DVAE}
M.~Zhang, S.~Jiang, Z.~Cui, R.~Garnett, and Y.~Chen, ``{D-VAE:} {A} variational
  autoencoder for directed acyclic graphs,'' in \emph{Advances in Neural
  Information Processing Systems 32: Annual Conference on Neural Information
  Processing Systems 2019, NeurIPS 2019, 8-14 December 2019, Vancouver, BC,
  Canada}, 2019, pp. 1586--1598.

\bibitem{neufeld1993pearl}
J.~Pearl, \emph{Probabilistic reasoning in intelligent systems: networks of
  plausible inference}.\hskip 1em plus 0.5em minus 0.4em\relax Morgan Kaufmann
  series in representation and reasoning, 1988.

\bibitem{CPS}
P.~Spirtes, C.~Glymour, and R.~Scheines, \emph{Causation, Prediction, and
  Search}.\hskip 1em plus 0.5em minus 0.4em\relax {MIT} Press, 2000.

\bibitem{Meek1995Causal}
C.~Meek, ``Causal inference and causal explanation with background knowledge,''
  \emph{Proc. Conf. on Uncertainty in Artificial Intelligence (UAI-95)}, pp.
  403--410, 1995.

\bibitem{AliferisTSB03}
C.~F. Aliferis, I.~Tsamardinos, A.~R. Statnikov, and L.~E. Brown, ``Causal
  explorer: {A} causal probabilistic network learning toolkit for biomedical
  discovery,'' in \emph{Proceedings of the International Conference on
  Mathematics and Engineering Techniques in Medicine and Biological Scienes,
  {METMBS}, June 23-26, Las Vegas, Nevada, {USA}}, 2003, pp. 371--376.

\end{thebibliography}

\begin{IEEEbiography}[{\includegraphics[width=1in,height=1.25in,clip,keepaspectratio]{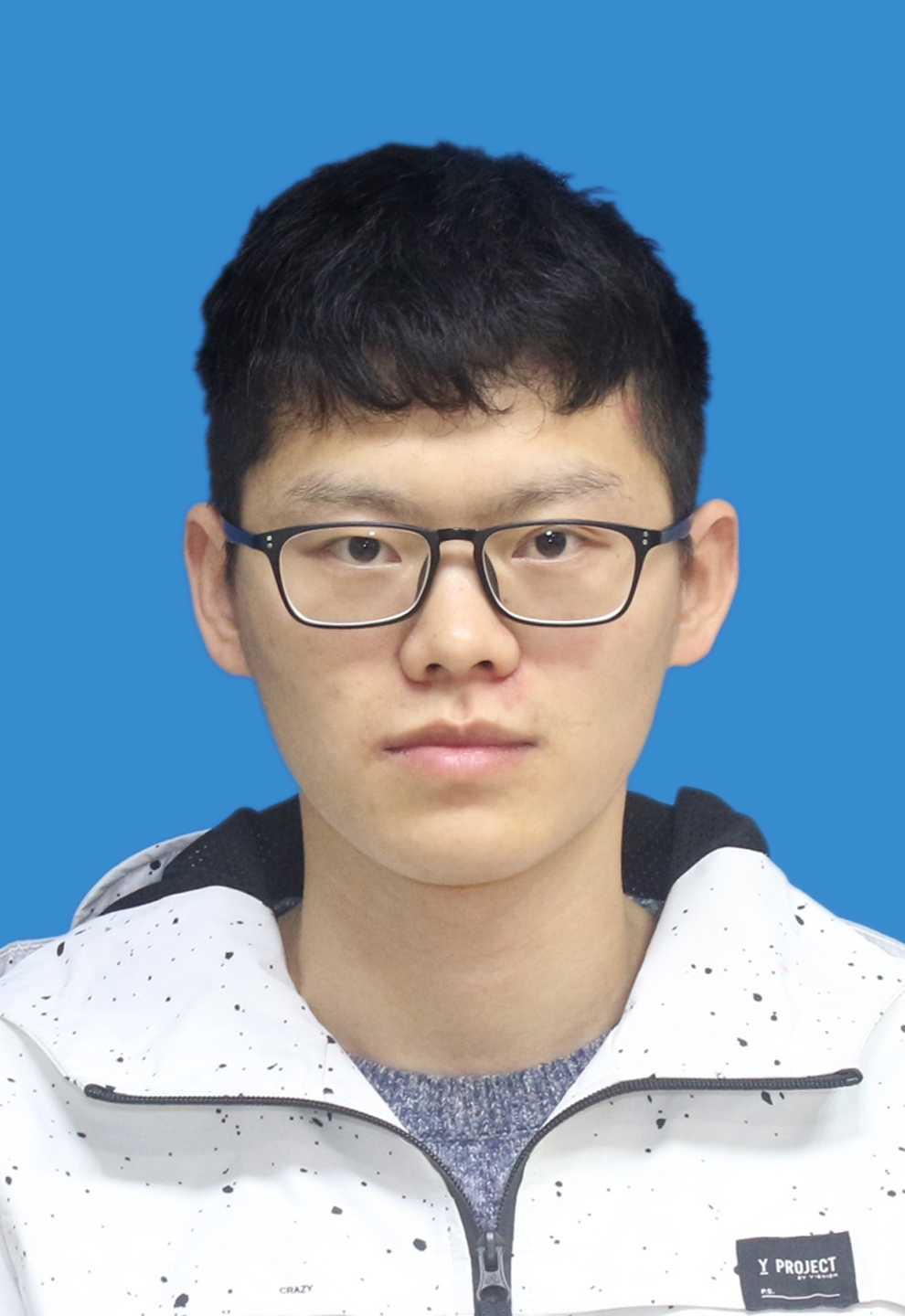}}]{Shuai Yang} received his B.S. and M.S. degrees in computer science from the Hefei University of Technology, Hefei, China, in 2016 and 2019, respectively, where he  is currently pursuing the Ph.D. degree with the School of Computer Science and Information Engineering. His main research interests include domain adaptation and  causal discovery.
\end{IEEEbiography}

\begin{IEEEbiography}[{\includegraphics[width=1in,height=1.25in,clip,keepaspectratio]{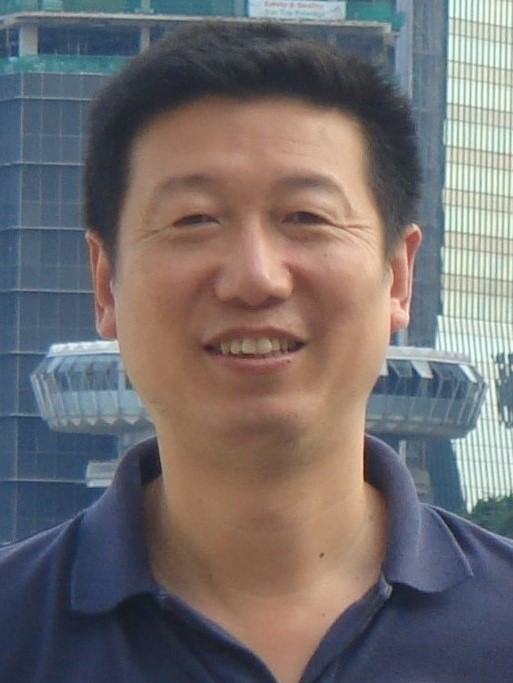}}]{Hao Wang} received the B.S. degree from the Department of Electrical Engineering and Automation, Shanghai Jiao Tong University, Shanghai, China, in 1984, and the M.S. and Ph.D. degrees in Computer Science from the Hefei University of Technology, Hefei, China, in 1989 and 1997, respectively. He is a Professor with the School of Computer Science and Information Engineering, Hefei University of Technology. His current research interests include artificial intelligence and robotics and knowledge engineering.
\end{IEEEbiography}

\begin{IEEEbiography}[{\includegraphics[width=1in,height=1.25in,clip,keepaspectratio]{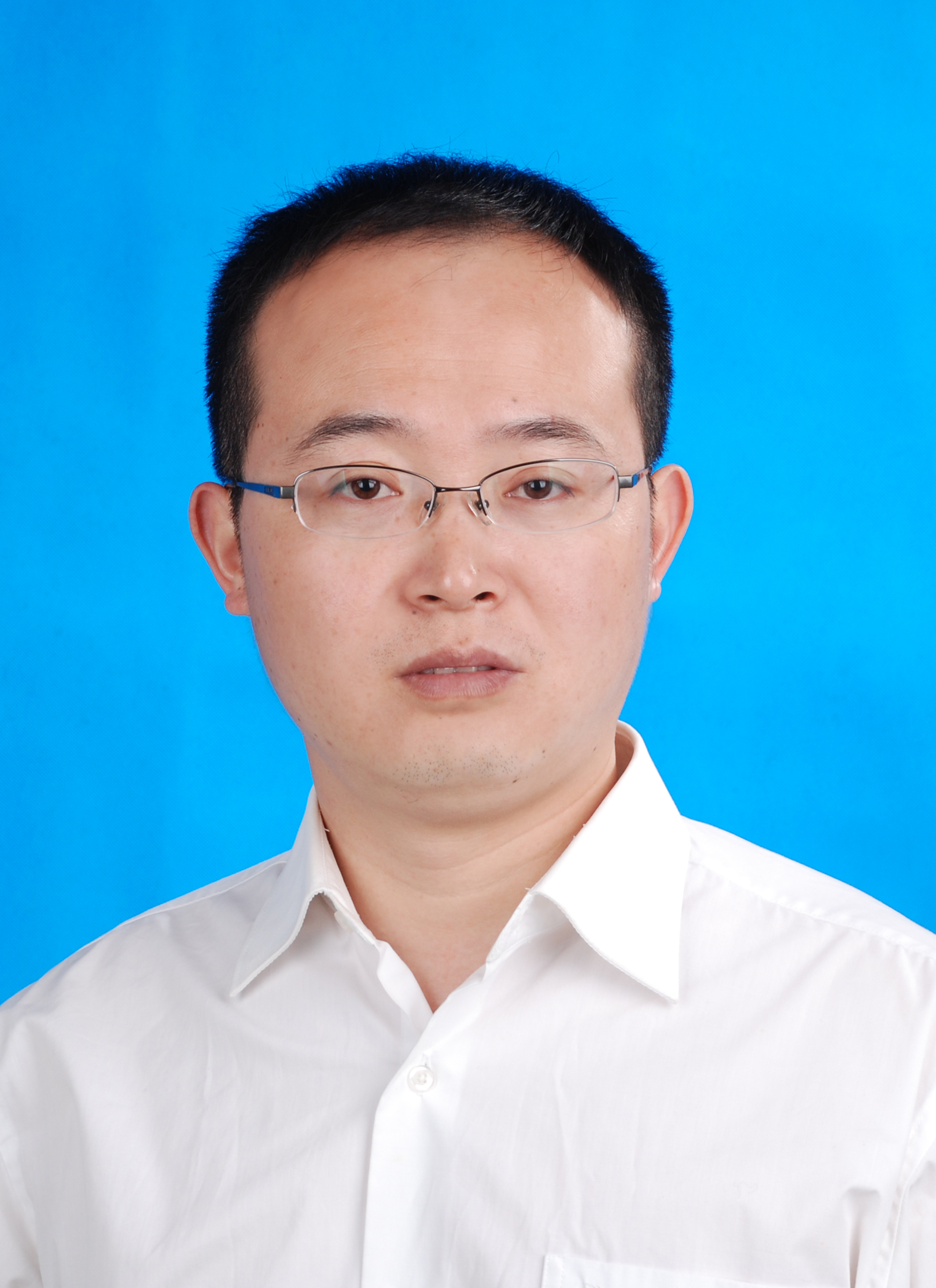}}]{Kui Yu} received the Ph.D. degree in computer science from  the Hefei University of Technology, Hefei, China, in 2013. He is currently a Professor with the School of Computer Science and Information Engineering, Hefei University of Technology. From 2015 to 2018, he was a Research Fellow of Computer Science with STEM, University of South Australia, Adelaide SA, Australia. From 2013 to 2015, he was a Postdoctoral Fellow with the School of Computing Science, Simon Fraser University, Burnaby, BC, Canada. His main research interests include causal discovery and machine learning.
\end{IEEEbiography}

\begin{IEEEbiography}[{\includegraphics[width=1in,height=1.25in,clip,keepaspectratio]{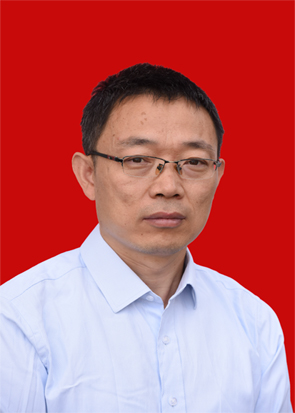}}]{Fuyuan Cao}  received the M.S. and Ph.D. degrees in computer science from the Shanxi University, Taiyuan, China, in 2004 and 2010, respectively. He is currently a professor with the School of Computer
and Information Technology, Shanxi University. His current research interests include machine learning  and clustering analysis.
\end{IEEEbiography}

\begin{IEEEbiography}[{\includegraphics[width=1in,height=1.25in,clip,keepaspectratio]{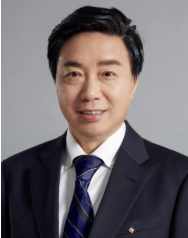}}]{Xindong Wu}(Fellow, IEEE) received the Ph.D. degree in artificial intelligence from the University of Edinburgh, Edinburgh, U.K., in 1993. He is currently a Chang Jiang Scholar with the School of Computer Science and Information Engineering, Hefei University of Technology, China, and also the Chief Scientist with the Mininglamp Academy of Sciences, Mininglamp Technology, Beijing, China. His research interests include data mining, knowledge-based systems, and web information exploration. He is a Fellow of the AAAS. He is also the Steering Committee Chair of the IEEE International Conference on Data Mining (ICDM), the Editor-in-Chief of \emph{Knowledge and Information Systems} and of Springer book series, \emph{Advanced Information and Knowledge Processing}.
\end{IEEEbiography}

\end{sloppypar}
\end{document}